\documentclass{article}
\usepackage{{James_arXiv}}
\usepackage{setspace}
\usepackage{apacite}

\usepackage[textwidth=3cm, disable]{todonotes}
\usepackage[normalem]{ulem}
\usepackage{algorithm, algorithmic}

\theoremstyle{plain}
\newtheorem{mythm}{Theorem}
\newtheorem{mydef}{Definition}
\newtheorem{assump}{Assumption}
\newtheorem{myprop}{Proposition}
\newtheorem{mylemma}{Lemma}
\newtheorem{mycoro}{Corollary}

\theoremstyle{definition}

\newcommand{\calN}{\mathcal{N}}
\newcommand{\calO}{\mathcal{O}}
\newcommand{\calC}{\mathcal{C}}
\newcommand{\calL}{\mathcal{L}}
\newcommand{\diag}{\mathrm{diag}}

\newcommand{\rref}[1]{\ref{#1}}
\newcommand{\eqrref}[1]{\eqref{#1}}

\hypersetup{
    colorlinks = true,
    linkbordercolor = {white},
    citecolor = {blue}
}

\title{Kalman Gradient Descent: Adaptive Variance Reduction in Stochastic Optimization}
\author{James Vuckovic (\url{james@jamesvuckovic.com})}

\begin{document}
	\maketitle
	\begin{abstract}
		We introduce Kalman Gradient Descent, a stochastic optimization algorithm that uses Kalman filtering to adaptively reduce gradient variance in stochastic gradient descent by filtering the gradient estimates. We present both a theoretical analysis of convergence in a non-convex setting and experimental results which demonstrate improved performance on a variety of machine learning areas including neural networks and black box variational inference. We also present a distributed version of our algorithm that enables large-dimensional optimization, and we extend our algorithm to SGD with momentum and RMSProp.
	\end{abstract}
	\begin{spacing}{0.1}
	\tableofcontents
	\end{spacing}
	\section{Introduction}\label{sec:intro}
	Stochastic optimization is an essential component of  most state-of-the-art the machine learning techniques. Sources of stochasticity in machine learning optimization include handling large datasets, approximating expectations, and modelling uncertain dynamic environments. The seminal work of \cite{robbins1985stochastic} showed that, under certain conditions, it is possible to use gradient-based optimization in the presence of randomness. However, it is well-known that gradient randomness has an adverse effect on the performance of stochastic gradient descent (SGD) \cite{wang2013variance}. As a result, the construction of methods to reduce gradient noise is an active field of research \cite{wang2013variance,NIPS2014_5557,grathwohl2017backpropagation,roeder2017sticking}.\newline

	Below, we propose a method using the celebrated Kalman filter \cite{Kalman1960} to reduce gradient variance in stochastic optimization in a way that is independent of the application area. Moreover, our method can be combined with existing (possibly application-specific) gradient variance-reduction methods. The specific contributions of this paper are:
	\begin{itemize}
		\item A novel framework for performing linear filtering of stochastic gradient estimates in SGD;
		\item Analysis of the asymptotic properties of the filter and a proof of convergence for the resulting optimization algorithm;
		\item Extensions of the proposed framework to modern optimization algorithms, and analysis of asymptotic filter properties in these cases;
		\item A novel, distributed variant of these algorithms to deal with high-dimensional optimization;
		\item Experiments comparing our algorithm to traditional methods across different areas of machine learning, demonstrating improved performance.
	\end{itemize}

	The remainder of this paper is organized as follows: In Section~\rref{sec:setup}, we set up the stochastic gradient descent algorithm as a linear system and construct the Kalman filter for this setup. In Section~\rref{sec:analysis} we provide a theoretical analysis of the filter and the proposed optimization algorithm. In Section~\rref{sec:extensions} we show how to extend this work and analysis to SGD with momentum and RMSProp, and propose a distributed variant of the algorithm suitable for large-scale optimization. In Section~\rref{sec:related} we connect our method to other areas of the machine learning literature, and in Section~\rref{sec:experiments} we apply these techniques to a variety of problems in machine learning. Finally, we discuss our conclusions in Section~\rref{sec:conclusion}.

	\section{Problem Setup}\label{sec:setup}
	We consider the problem  
	\[
		\min_{x\in\R^n} f(x),
	\]where $f:\R^n\to \R$ is assumed to be at least differentiable, using stochastic gradient methods. Following the notation in \cite{bottou2018optimization}, we will assume that we cannot directly observe $\nabla f(x)$, but instead we can evaluate a function $g:\R^n\times \R^k\to \R^n$ with ${x,\xi\mapsto g(x;\xi)}$ where $\xi$ is a $\R^k$-valued random variable and $g$ satisfies
	\[
		E_\xi[g(x,\xi)]=\nabla f(x).
	\]In other words, in our setup we cannot compute the gradient directly, but we can obtain a stochastic, unbiased estimate of the gradient.\newline

	Define a discrete-time stochastic process $\{x_t\}$ by
	\begin{equation}\label{eq:sgd}
		x_{t+1}=x_t-\alpha_t g(x_t,\xi_t)
	\end{equation}
	where $\{\xi_t\}$ is a sequence of i.i.d. realizations of $\xi$, $\{\alpha_t\}$ is a sequence of non-negative scalar stepsizes, and $x_0$ is arbitrary. This is a stochastic approximation of the true gradient descent dynamics $x_{t+1}=x_t-\alpha_t \nabla f(x_t)$. We will investigate how to set up this equation as a discrete-time, stochastic, linear dynamical system, and then apply linear optimal filtering to \eqrref{eq:sgd}.

	\subsection{Linear Dynamics}\label{sec:dyn}

	The update \eqrref{eq:sgd} is a linear equation in variables $x_t$ and $g(x_t,\xi_t)$. We will set up this equation in such a way that it can be filtered with linear filtering methods. Hence consider the discrete-time, stochastic, linear time-varying (LTV) system
	\begin{equation}\label{eq:sys1}
		\begin{bmatrix}
			x_{t+1} \\ g_{t+1}
		\end{bmatrix}=\begin{bmatrix}
			I & -\alpha_t \\ 0 & I
		\end{bmatrix}\begin{bmatrix}
			x_t \\ g_t
		\end{bmatrix} + w_t
	\end{equation}
	where $w_t\sim \calN(0,Q_t)$ with $Q_t=\sigma_QI_{2n\times 2n}~,\sigma_Q>0$. When working with block matrices as above, we use the convention $aI=a$ here and henceforth. Here, $g_t$ represents a hidden estimate of $\nabla f(x_t)$, not necessarily $g(x_t,\xi_t)$ as above. This state-space setup is commonly called the ``local linear model'' in time-series filtering \cite{sarkka2013bayesian}. \newline

	For the purposes of filtering, we must include a set of measurement equations for \eqrref{eq:sys1}. We propose
	\begin{equation}\label{eq:sys2}
		y_t=\begin{bmatrix}
			0  & I 
		\end{bmatrix}\begin{bmatrix}
			x_t \\ g_t
		\end{bmatrix} + v_t
	\end{equation}
	with $v_t\sim \calN(0,R_t)$ where $R_t=\sigma_{R}I_{n\times n},~\sigma_R>0$. We will use this measurement equation to model $g(x_t;\xi_t)$ by
	\begin{equation}
		g(x_t;\xi_t)= y_t = g_t + v_t
	\end{equation}
	where $g_t$ will be estimated by the Kalman filter. In this way, $g(x_t; \xi_t)$ is implicitly modelled as $g(x_t; \xi_t)\approx\nabla f(x_t)+ v_t$.\newline

	By making the following identifications
	\begin{equation}\label{eq:components}
		z_t :=\begin{bmatrix}
			x_t \\ g_t
		\end{bmatrix},~~~A_t:=\begin{bmatrix}
			I & -\alpha_t \\ 0 & I
		\end{bmatrix},~~~C_t:=\begin{bmatrix}
			0  & I 
		\end{bmatrix}
	\end{equation}
	we will often abbreviate the system \eqrref{eq:sys1},\eqrref{eq:sys2} to
	\begin{equation}\label{eq:linsys}
		\rl{
			z_{t+1} & = A_t z_t + w_t\\
			y_t & = C_t z_t + v_t
			}.
	\end{equation}

	It is important to note that the trajectories in \eqrref{eq:linsys} are not uniquely determined given $y_0$. This is due to the absence of the initial value $x_0$ in our measurement equation. In reality, we will always know $x_0$ (it is required by the algorithm) so we must simply modify $C_t$ to reflect this by setting $C_0=I_{2n\times 2n}$ and $C_t=[0~~I]$ for $t>0$, and $R_0=I_{2n\times 2n}$ and $R_t=R$ otherwise.

	\subsection{Kalman Filtering}
	We will now develop the Kalman filter for the LTV stochastic dynamical system \eqrref{eq:linsys}. We would like to compute $\wat{z}_{t|t}:=E[z_t|y_1,\dots,y_t]$ at every timestep $t$. The Kalman filter is a deterministic LTV system for computing $\wat{z}_{t|t}$ in terms of $\wat{z}_{t|t-1}=E[z_t|y_1,\dots,y_{t-1}]$ and the observation $y_t$. The filter also propagates the covariance matrix of the error measurement $P_{t|t}=E[(z_t-\wat{z}_{t|t})(z_t-\wat{z}_{t|t})^T]$ and similarly for $P_{t|t-1}$.\newline

	The Kalman filter consists of a set of auxiliary (deterministic) LTV dynamics for $t>0$ \cite{jazwinski2007stochastic}
	\begin{align}
		\wat{z}_{t|t-1}&=A_{t-1}\wat{z}_{t-1|t-1}\label{eq:kal1}\\
		P_{t|t-1}&=A_{t-1}P_{t-1|t-1}A_{t-1}^T+Q_t\label{eq:kal2}\\
		\wat{y}_t&=y_t-C_t \wat{z}_{t|t-1}\label{eq:kal3}\\
		K_t&=P_{t|t-1}C_t^T(R_t+C_tP_{t|t-1}C_t^T)^{-1}\label{eq:kal4}\\
		\wat{z}_{t|t}&=\wat{z}_{t|t-1}+K_t\wat{y}_t\label{eq:kal5}\\
		P_{t|t}&=(I-K_tC_t)P_{t|t-1}\label{eq:kal6}
	\end{align}
	with initial values $\wat{z}_{0|0} = \wat{z}_{0}$ and $P_{0|0} \equiv P_{0}$. It is well-known that the Kalman filter dynamics produce the optimal (i.e. minimum variance) linear estimate of $z_t$ \cite{jazwinski2007stochastic}.\newline

	Writing $\wat{z}_{t|t}=[\wat{x}_{t|t}^T,~\wat{g}_{t|t}^T]^T$; $\wat{g}_{t|t}$ is then the optimal linear estimate of $g_t$ given the observations $y_1,\dots,y_t$. Let us rewrite $\wat{g}_{t|t}$ using equations \eqrref{eq:kal1}-\eqrref{eq:kal6}, noting that
	\begin{equation}
		\wat{z}_{t|t-1}=A_{t-1}\wat{z}_{t-1|t-1}= A_{t-1} \begin{bmatrix}
			\wat{x}_{t-1|t-1}\\ \wat{g}_{t-1|t-1}
		\end{bmatrix} =\begin{bmatrix}
			\wat{x}_{t|t-1}\\ \wat{g}_{t-1|t-1}
		\end{bmatrix},
	\end{equation}
	by multiplying $\wat{z}_{t|t}$ by $C_t$ to obtain
	\begin{align}
		\wat{g}_{t|t} &= C_t \wat{z}_{t|t} = C_t A_{t-1}\wat{z}_{t-1|t-1} + C_tK_t (y_t - C_t A_{t-1}\wat{z}_{t-1|t-1})\\
		&= \wat{g}_{t-1|t-1} + C_tK_t(y_t - \wat{g}_{t-1|t-1})\\
		&= (I-C_t K_t) \wat{g}_{t-1|t-1} + C_t K_t y_t\\
		&= (I-\wtilde{K}_t) \wat{g}_{t-1|t-1} + \wtilde{K}_ty_t
	\end{align}
	where 
	\begin{equation}
		\wtilde{K}_t = C_t K_t = C_tP_{t|t-1}C_t^T(R_t+C_tP_{t|t-1}C_t^T)^{-1}.
	\end{equation} 
	If $P_{t|t-1}$ is a uniformly bounded, positive definite matrix, it is easy to see that $\wtilde{K}_t$ is a bounded positive definite matrix s.t. $\exists a,b>0$ s.t. $0<aI <\wtilde{K}_t<b I <I~\forall t\geq 0$. Intuitively, $K_t$ adapts depending on the uncertainty of the estimate $\wat{z}_{t|t-1}$ relative to the measurement uncertainty $R_t$.\newline

	Hence we see that $\wat{g}_{t|t}$ is an exponentially smoothed version of $g(x_t;\xi_t)$ where $\wtilde{K}_t$ is an adaptive smoothing matrix. We will use this estimate $\wat{g}_{t|t}$ as a ``better'' approximation for $\nabla f(x_t)$ than $g(x_t,\xi_t)$ in \eqrref{eq:sgd}. Writing $v_{t+1}:=\wat{g}_{t|t}$ (not to be confused with the measurement noise term from before), we will study the properties of the update
	\begin{equation}\label{eq:kgd}
		\rl{
			v_{t+1} &= (I-\wtilde{K}_t) v_t + \wtilde{K}_tg(x_t;\xi_t)\\
			x_{t+1} &= x_t -\alpha_t v_{t+1}
		}
	\end{equation}
	which we call the \textbf{Kalman gradient descent} (KGD) dynamics.
	
	\begin{algorithm}
		\caption{Kalman Gradient Descent}
		\begin{algorithmic}
			\REQUIRE $x_0,~z_0,~P_0,~g(\cdot,\cdot),~T,~\sigma_Q,~\sigma_R,~\{\alpha_t\}$
			\STATE Initialize the Kalman Filter $KF$ with $z_0,P_{0},~Q=\sigma_QI,~R=\sigma_RI,~C$
			\FOR{$t=0,\dots,T-1$}
				\STATE $dx \leftarrow g(x_t,\xi_t)$.
				\STATE Assemble $A_t$.
				\STATE Increment Kalman Filter $\wat{z}_{t|t}\leftarrow KF(dx,A_t)$.
				\STATE Extract $v_{t+1}\leftarrow \wat{g}_{t|t}$ from $\wat{z}_{t|t}$.
				\STATE $x_{t+1}\leftarrow x_t - \alpha_t v_{t+1}$.
			\ENDFOR
			\RETURN $x_{T}$
		\end{algorithmic}
		\label{alg:kgd}
	\end{algorithm}
	It is important to note that this setup is \emph{not} equivalent to the ``heavy-ball'' momentum method, even in the case that $\wtilde{K}_t\equiv\beta$ is a constant scalar. See the remarks in Appendix~\rref{app:misc} for a calculation that shows that the two methods cannot be made equivalent by a change of parameters.

	\section{Analysis}\label{sec:analysis}
	The analysis of the trajectories defined by \eqrref{eq:kgd} is broken into two components: first, we study of the filter asymptotics (stability, convergence, and robustness); and second, we study the the convergence of \eqrref{eq:kgd} to a stationary point of $f$. 

	\subsection{Filtering}
	By using the Kalman filter estimate $\wat{g}_{t|t}$ instead of $g(x_t,\xi_t)$ or indeed $\nabla f(x_t)$, we lose any \emph{a~priori} guarantees on the behaviour of the gradient estimate (e.g. boundedness). Since these guarantees are usually required for subsequent analysis, we must show that the filtered estimate is, in-fact, well-behaved. More precisely, we will show that the linear dynamics $\wat{z}_{t|t}$ are stable, have bounded error (in the $L^2$-sense), and are insensitive to mis-specified initial conditions.\newline

	The general conditions under which these good properties hold are well-studied in LTV filtering. We defer the majority of these domain-specific details to the Appendix~\rref{app:ltv} and \cite{jazwinski2007stochastic} while stating the results in-terms of our stochastic optimization setup below. We will need the following definition.

	\begin{mydef}[Stability]
		Let $m_{t+1}=F_t m_t + B_tu_t$ be an arbitrary controlled LTV system in $\R^n$. Let $\Phi(t,s)$ be the solution operator for the homogeneous equation $m_{t+1}=F_t m_t$. Such a system is called \textbf{internally asymptotically stable} if $\exists c_0,c_1\geq 0$ s.t. 
		\[
			\|\Phi(t,0)m_0\|\leq c_1\exp(-c_2 t)\|m_0\|
		\]for any $m_0\in\R^n$. A system which is internally asymptotically stable is also \textbf{BIBO stable} (bounded-input bounded-output stable), in the sense that a sequence of bounded inputs $u_{1},\dots,u_{t}$ will produce a sequence of bounded states.
	\end{mydef}
	
	We can rewrite the state estimate $\wat{z}_{t|t}$ dynamics as
	\begin{equation}\label{eq:oneline}
		\wat{z}_{t|t}= P_{t|t}P_{t|t-1}^{-1}A_{t-1}\wat{z}_{t-1|t-1} + P_{t|t}C_t^TR_t^{-1}y_t
	\end{equation}
	which is a controlled linear system with inputs $y_t$. Hence, when we refer to BIBO stability of the Kalman filter, we are saying that \eqrref{eq:oneline} is BIBO stable.\newline

	We will use a partial ordering on $n\times n$ real matrices by saying that $A<B$ iff $B-A$ is positive definite, and $A\leq B$ iff $B-A$ is positive semidefinite. In the sequel, we will also maintain the following assumption on $\alpha_t$ which is the same as in \cite{robbins1985stochastic}.

	\begin{assump}\label{assump:alphamain}
		$\{\alpha_t\}_{t=0}^\infty\subset\R$ is non-increasing, $\alpha_t>0~\forall t$, $\sum \alpha_t= \infty,~\sum \alpha_t^2<\infty$.
	\end{assump}

	\begin{mythm}[Filter Asymptotics]\label{thm:filt}
		Suppose Assumption~\rref{assump:alphamain} holds and that $\wat{z}_{t|t}$ and $P_{t|t}$ are governed by the Kalman filter equations \eqrref{eq:kal1}-\eqrref{eq:kal6}. Then:
		\begin{enumerate}[(a)]
			\item \emph{(Stability)} The filtered dynamics \eqrref{eq:oneline} are internally asymptotically stable, hence BIBO stable;
			\item \emph{(Bounded Error Variance)} If $P_0>0$ there exists $\rho\in\R_{>0}$ and $N\in\Z_{>0}$ s.t. 
			\[
				\frac{1}{\rho}I\leq P_{t|t}\leq \rho I~~~\forall t\geq N;
			\]
			\item \emph{(Robustness)} Let $P^1_{t|t},P^2_{t|t}$ be two solutions to the Kalman filter dynamics with initial conditions $P^1_0,P^2_0\geq 0$ resp. Then $\exists k_1,k_2\in\R_{\geq 0}$ s.t. 
			\[
				\|P^1_{t|t}-P^2_{t|t}\|\leq k_1 e^{-k_2 t}\|P^1_0-P^2_0\|\to 0
			\]as $t\to \infty$.
		\end{enumerate}
	\end{mythm}
	\begin{proof}
		Using Lemma~\rref{lem:ctrlobs} in Appendix~\rref{app:ltv}, apply Theorems~\rref{thm:kfstab},~\rref{thm:kfbdd},~and~\rref{thm:kfrob} respectively.
	\end{proof}

	\subsection{Optimization}
	We now study the convergence properties of the KGD dynamics \eqrref{eq:kgd}. We first assume some conditions on $f$, then prove convergence. In the sequel, $E[\cdot]$ will denote expectation w.r.t the joint distribution of all $\xi_s$ that appear in the expectation.
	\begin{assump}\label{assump:regmain}
		The objective function $f:\R^n\to \R$ is $C^3$, with uniformly bounded 1st, 2nd, and 3rd derivatives. In particular $\nabla f$ is Lipschitz with constant $L$.
	\end{assump}

	\begin{assump}[\cite{bottou2018optimization}]\label{assump:gmain}
		The random variables $g(x_t;\xi_t)$ satisfy the following properties:
		\begin{enumerate}[(a)]
			\item $E[\nabla f(x_t)^T g(x_t;\xi_t)]\geq \mu E[\|\nabla f(x_t)\|^2]$ for some constant $\mu>0$ and $\forall t$;
			\item $E[\|g(x_t;\xi_t)\|^2]\leq M + M_GE[\|\nabla f(x_t)\|^2]$ for constants $M,~M_G>0$ $\forall t$.
		\end{enumerate}
	\end{assump}

	\begin{mythm}[Convergence of KGD]\label{thm:kgdconv}
		Assume that Assumptions~\rref{assump:alphamain},~\rref{assump:regmain},~and \rref{assump:gmain} hold. Then
		\[
			\liminf_{t\to \infty} E[\|\nabla f(x_t)\|^2]=0
		\]where $x_t$ evolves according to \eqrref{eq:kgd}.
	\end{mythm}
	\begin{proof}
		This follows from Corollary~\rref{coro:result} and Proposition~\rref{prop:reduction} in Appendix~\rref{app:opt}.
	\end{proof}
	The proof of this result follows the same steps as in \cite{bottou2018optimization}~Theorem 4.10, but in our case we must account for the fact that the smoothed estimate $v_{t+1}$ is not a true ``descent direction''. However, if $\nabla f$ varies ``slowly enough'' then the smoothing error grows sufficiently slowly to allow for convergence. In practice, we see that the benefits of reduced variance greatly outweigh the drawbacks of using an approximate direction of descent.

	\subsection{Scalability}
	The main drawback of this algorithm is that the Kalman filter requires a series of matrix multiplications and inversions. In most implementations, these are $\calO(d^{2.807})$ for a $d\times d$ matrix \cite{CLRS}. Moreover, we require $\calO(d^2)$ extra space to store these matrices. Fortunately, there are GPU-accelerated Kalman filtering algorithms available \cite{huang2011accelerating} which improve this bottleneck. Also, in Section~\rref{sec:extensions}, we will introduce a distributed version of Algorithm~\rref{alg:kgd} that specifically deals with the issue of high-dimensionality and enables horizontal scaling of the KGD algorithm in addition to the vertical scaling described here.

	\section{Extensions}\label{sec:extensions}
	We consider two types of extensions to the KGD algorithm: extending the setup to momentum and RMSProp, and a distributed version of KGD that addresses the scalability concerns of high-dimensional matrix operations.

	\subsection{Momentum \& RMSProp}
	We study extensions of the KGD filtering setup to two modern optimization algorithms: SGD with momentum \cite{qian1999momentum} and RMSProp \cite{Tieleman2012}.\newline

	Consider the momentum update \cite{qian1999momentum}
	\begin{equation}\label{eq:mom0}
		\rl{
			x_{t+1} & = x_t + \alpha_t u_{t+1}\\
			u_{t+1} & = \mu_t u_t - (1-\mu_t) \nabla f(x_t)
		} \iff 
		\rl{
			x_{t+1} & = x_t + \alpha_t \mu_t u_t - \alpha_t(1-\mu_t) \nabla f(x_t)\\
			u_{t+1} & = \mu_t u_t - (1-\mu_t) \nabla f(x_t)
		}
	\end{equation}
	with $0<\mu_t<1$ and $\alpha_t>0$.
	Rewriting these dynamics in the style of \eqrref{eq:sys1}-\eqrref{eq:sys2} (and including the additive noise terms as before) we have a LTV representation of the momentum update which can be filtered:
	\begin{align}
		\begin{bmatrix}
			x_{t+1} \\ u_{t+1} \\ g_{t+1}
		\end{bmatrix}&=\begin{bmatrix}
			I & \alpha_t \mu_t  & -\alpha_t (1-\mu_t )\\
			0 & \mu_t  & -(1-\mu_t) \\
			0 & 0 & I
		\end{bmatrix}\begin{bmatrix}
			x_{t} \\ u_{t} \\ g_{t}
		\end{bmatrix} + w_t\label{eq:mom1}\\
		y_t & = \begin{bmatrix}
			0 & 0 & I
		\end{bmatrix}\begin{bmatrix}
			x_{t} \\ u_{t} \\ g_{t}
		\end{bmatrix} + v_t.\label{eq:mom2}
	\end{align}

	In a similar fashion to momentum, consider the RMSProp update \cite{Tieleman2012}
	\begin{equation}\label{eq:rms0}
		\rl{
			x_{t+1} & = x_t - \alpha_t \diag(\beta_t) \nabla f(x_t)\\
			r_{t+1} & = \gamma_t r_t + (1-\gamma_t)\diag(\nabla f(x_t))\nabla f(x_t)
		}
	\end{equation}
	with $\beta_t=(\sqrt{r_{t+1}}+\veps)^{-1}=((\gamma_t r_t + (1-\gamma_t)\diag(\nabla f(x_t))\nabla f(x_t))^{-1/2} + \veps)$ which we rewrite as

	\begin{align}
		\begin{bmatrix}
			x_{t+1} \\ r_{t+1} \\ g_{t+1}
		\end{bmatrix}&=\begin{bmatrix}
			I & 0  & -\alpha_t \diag(\beta_t)\\
			0 & \gamma_t  & (1-\gamma_t)\diag(\nabla f(x_t)) \\
			0 & 0 & I
		\end{bmatrix}\begin{bmatrix}
			x_{t} \\ r_{t} \\ g_{t}
		\end{bmatrix} + w_t\label{eq:rms1}\\
		y_t & = \begin{bmatrix}
			0 & 0 & I
		\end{bmatrix}\begin{bmatrix}
			x_{t} \\ r_{t} \\ g_{t}
		\end{bmatrix} + v_t\label{eq:rms2}
	\end{align}
	It is important to note that \eqrref{eq:rms1} does not correspond exactly to a realistic setup because we assume $\nabla f(x_t)$ is used to construct the transition matrix, whereas in practice we will only have access to stochastic estimates of these quantities via $g(x_t;\xi_t)$. Dealing in-detail with random transition matrices is beyond the scope of this investigation. In the experiments below, we have always used whichever gradient estimate was provided to the optimization algorithm (i.e. $g(x_t;\xi_t)$) to construct the transition matrix.

	\begin{myprop}
		Assume that Assumptions~\rref{assump:alphamain} and \rref{assump:gmain} hold. If $\mu_t=\mu$ in \eqrref{eq:mom1} and $\gamma_t=\gamma$ in \eqrref{eq:rms1} are constant, then the Kalman filter dynamics for \eqrref{eq:mom1}-\eqrref{eq:mom2} and \eqrref{eq:rms1}-\eqrref{eq:rms2} are stable, have bounded error, and are robust in the sense of Theorem~\rref{thm:filt}.
	\end{myprop}
	\begin{proof}
		Use Lemma~\rref{lem:extobsctrl} in Appendix~\rref{app:ltv} to apply Theorem~\rref{thm:kfstab},~\rref{thm:kfbdd},~\rref{thm:kfrob} respectively.
	\end{proof}
	We see that the KGD algorithm can be easily adapted to include these more sophisticated updates, with the filtering step being adjusted according to the linear state-space model being used. In fact, the principle of pre-filtering gradients before using them in optimization is applicable to most optimization algorithms, such as AdaGrad \cite{duchi2011adaptive} or Adam \cite{kingma2014adam}.

	\subsection{Distributed KGD}
	In this section, we present a distributed version of KGD that specifically targets issues with high-dimensional matrix operations. Indeed, as pointed out in Section~\rref{sec:analysis}, the Kalman filter uses matrix multiplications and inversions which have a cost of $\calO(d^{2.807})$ $d\times d$-matrices \cite{CLRS}. This makes dimensionality a very real concern, since machine learning applications may use hundreds of thousands or millions of parameters.\newline

	To combat this, we propose a ``divide and conquer'' variant of KGD (which applies \emph{mutatis mutandis} to the variations developed above) that splits the parameter vectors of dimension $d$ into $N_D$ sub-vectors of dimension $D\ll d$ and runs separate synchronous optimizers on each sub-vector. This technique enables KGD to scale horizontally on a single machine or to several machines. 
	See Algorithm~\rref{alg:distkgd} for a precise description.\newline

	\begin{algorithm}
		\caption{Distributed Kalman Gradient Descent}
		\begin{algorithmic}
			\REQUIRE $x_0,~z_0,~P^{(i)}_0,~g(\cdot,\cdot),~T,~\sigma_Q,~\sigma_R,~\alpha_t,~D$
			\STATE Compute $N_D=\lceil n/D\rceil$ and split $z_0$ into $z^{(i)}_0,~i=1,\dots,N_D$.
			\STATE Initialize $N_D$ Kalman Filters $KF^{(i)}$ with $z^{(i)}_0,P^{(i)}_{0},~Q=\sigma_QI,~R=\sigma_RI,~C$.
			\FOR{$t=0,\dots,T-1$}
				\STATE $dx \leftarrow g(x_t,\xi_t)$.
				\STATE Split $dx$ into $dx^{(i)},~i=1,\dots,N_D$.
				\FOR{$i=1,\dots,N_D$}
					\STATE Assemble $A^{(i)}_t$.
					\STATE Increment Kalman Filter $\wat{z}^{(i)}_{t|t}\leftarrow KF^{(i)}(dx^{(i)},A^{(i)}_t)$.
					\STATE Extract $v^{(i)}_{t+1}\leftarrow \wat{g}^{(i)}_{t|t}$ from $\wat{z}^{(i)}_{t|t}$.
					\STATE $x^{(i)}_{t+1}\leftarrow x^{(i)}_t - \alpha_t v^{(i)}_{t+1}$.
				\ENDFOR
				\STATE Combine $x^{(i)}_{t+1},~i=1,\dots,N_D$ into $x_{t+1}$.
			\ENDFOR
			\RETURN $x_{T}$
		\end{algorithmic}
		\label{alg:distkgd}
	\end{algorithm}

	Assuming $d$ is divisible by $D$ for simplicity, we write $d=N_D\cdot D$. In the ordinary case we have $\calO((N_D\cdot D)^\gamma)$ cost and in the distributed case we have $\calO(N_D \cdot D^\gamma)$. This is a speedup of $\calO(N_D^{\gamma -1})$, where $\gamma$ is usually $\log_2 7\approx 2.807$.\newline

	This speedup is balanced by the facts that 1. in practice, the quality of the gradient filter is decreased by the sub-vector approximation, and 2. by the constant factors involved with the operation of $N_D$ independent filtering optimizers. Hence, a good strategy is to find the largest $D$ which produces acceptable runtime for the matrix operations, and then implement the Algorithm~\rref{alg:distkgd} above.\newline

	While considerable speedups are available from this technique when used on a single machine due to the reduced dimension (e.g. the experiment in Section~\rref{subsec:mnist}), it is also clear that Distributed KGD is amenable to a synchronous implementation on multiple cores or machines. Combined with GPU-accelerated Kalman filter implementations as described in Section~\rref{sec:analysis}, this distributed KGD framework is a potential candidate for large-scale optimization.

	\section{Related Work}\label{sec:related}

	The modern gradient-based stochastic optimization landscape has several key developments beyond pure gradient descent. These include momentum methods \cite{qian1999momentum, nesterov1983method, tseng1998incremental}, AdaGrad \cite{duchi2011adaptive}, RMSProp \cite{Tieleman2012}, and Adam \cite{kingma2014adam}. In particular, exponential moving averages are used by both KGD and Adam, though we use an adaptive scaling matrix in KGD instead of a constant factor (as in Adam) to control the averaging. Recently, a general framework for adaptive methods using exponential moving averages (which includes Adam) was presented in \cite{reddi2018convergence}; KGD fits into this framework as well. 
	\newline

	There have been a few previous examples of Kalman filters being applied to stochastic gradient descent \cite{bittner2004kalman, patel2016kalman, akyildiz2018incremental}. In \cite{bittner2004kalman}, the authors instead develop dynamics for the gradient and an approximation of the Hessian, omitting the state variable. In \cite{patel2016kalman}, the authors specialize to the case of large-scale linear regression. Stopping rules for optimization algorithms using Kalman filters are derived in these cases. In \cite{akyildiz2018incremental}, the incremental proximal method (IPM) is linked to the Kalman filter.\newline

	There is an important connection between KGD and meta-learning and adaptive optimization \cite{schraudolph1999local, andrychowicz2016learning}. In \cite{schraudolph1999local}, the authors propose a set of auxiliary dynamics for the stepsize $\alpha_t$, which is treated as a  vector of $n$ individual parameters. More recently, in \cite{andrychowicz2016learning}, the authors propose the update $x_{t+1}=x_t + g_t(\nabla f(x_t), \phi)$ where the function $g_t(\nabla f(\theta_t),\phi)$ is represented by a recurrent neural network and is learned during the optimization. \newline

	Our method can be considered a type of meta-learning similar to \cite{schraudolph1999local} in which the meta-learning dynamics are those of the Kalman filter. Our setup also relates  to \cite{andrychowicz2016learning} by restricting the function $g_t(\nabla f(\theta_t),\phi)$ to be linear (albeit with a different loss function that promotes approximation rather than direct optimization). In this context, the KGD algorithm learns an optimizer in a recursive, closed form that is gradient-free.\newline

	Lastly, we note that variance reduction techniques in SGD are an active domain of research, see \cite{wang2013variance}. Techniques such as control-variates \cite{wang2013variance, grathwohl2017backpropagation} and domain-specific reduced-variance gradient estimators \cite{roeder2017sticking} could be layered on top of our techniques for enhanced performance.

	\section{Experiments}\label{sec:experiments}

	To study the various properties of the family of KGD algorithms proposed thus far, we conducted several experiments. The first is a simple optimization problem to study in detail the behaviour of the KGD dynamics, the second is a Bayesian parameter inference problem using Black Box Variational Inference (BBVI), the third is a small-scale neural network regression problem, and the last is larger-scale distributed KGD training of a MLP MNIST classifier. Experiments 2-4 are based on some excellent examples of use-cases for Autograd \cite{autograd}. All experiments use $\sigma_Q=0.01,~\sigma_R=2.0,~P_0=0.01 I$.\footnote{Software implementations of all variations of KGD as well as these experiments can be found at \url{https://github.com/jamesvuc/KGD}.}

	\subsection{2D Stochastic Optimization}\label{sec:exp1}
	We tested our filtered optimization technique on a two-dimensional optimization problem. Specifically, we minimized
	\begin{equation}\label{eq:opt1}
		f(x^1,x^2)=0.1\left(\left(x^1\right)^2 + \left(x^2\right)^2\right) + \sin\left(x^1 + 2x^2\right)
	\end{equation}
	with gradient descent \eqrref{eq:sgd}, gradient descent with momentum \eqrref{eq:mom0}, and RMSProp \eqrref{eq:rms0} using:
	\begin{enumerate}
		\item The true gradient $\nabla f(x_t)$;
		\item A noisy gradient of the form $g(x_t,\xi_t)=\nabla f(x_t)+\xi_t,~\xi_t\sim \calN(0, I_{2\times 2})$; and
		\item  The Kalman filtered noisy gradient $\wat{g}_{t|t}$.
	\end{enumerate}
	The function $f$ is approximately ``bowl-shaped'' and has many suboptimal local minima which can trap the optimization dynamics. The results for the SGD dynamics are in Figure~\rref{fig:opt1}, with the momentum and RMSProp results in Figure~\rref{fig:addtl} of Appendix~\rref{app:imagery}. \newline

	In Figure~\rref{fig:opt1}, we see that the noiseless dynamics did indeed get trapped in a local minimum. For the noisy gradient, the gradient noise actually had a ``hill-climbing'' effect, but ultimately this too was trapped by a local minimum. We see that the filtered dynamics were able to avoid the ``bad'' local minima of the first two versions and overall performed much better in terms of minimization and stability.
	
	\begin{figure}[ht!]
	\centering
		\includegraphics[scale=0.4]{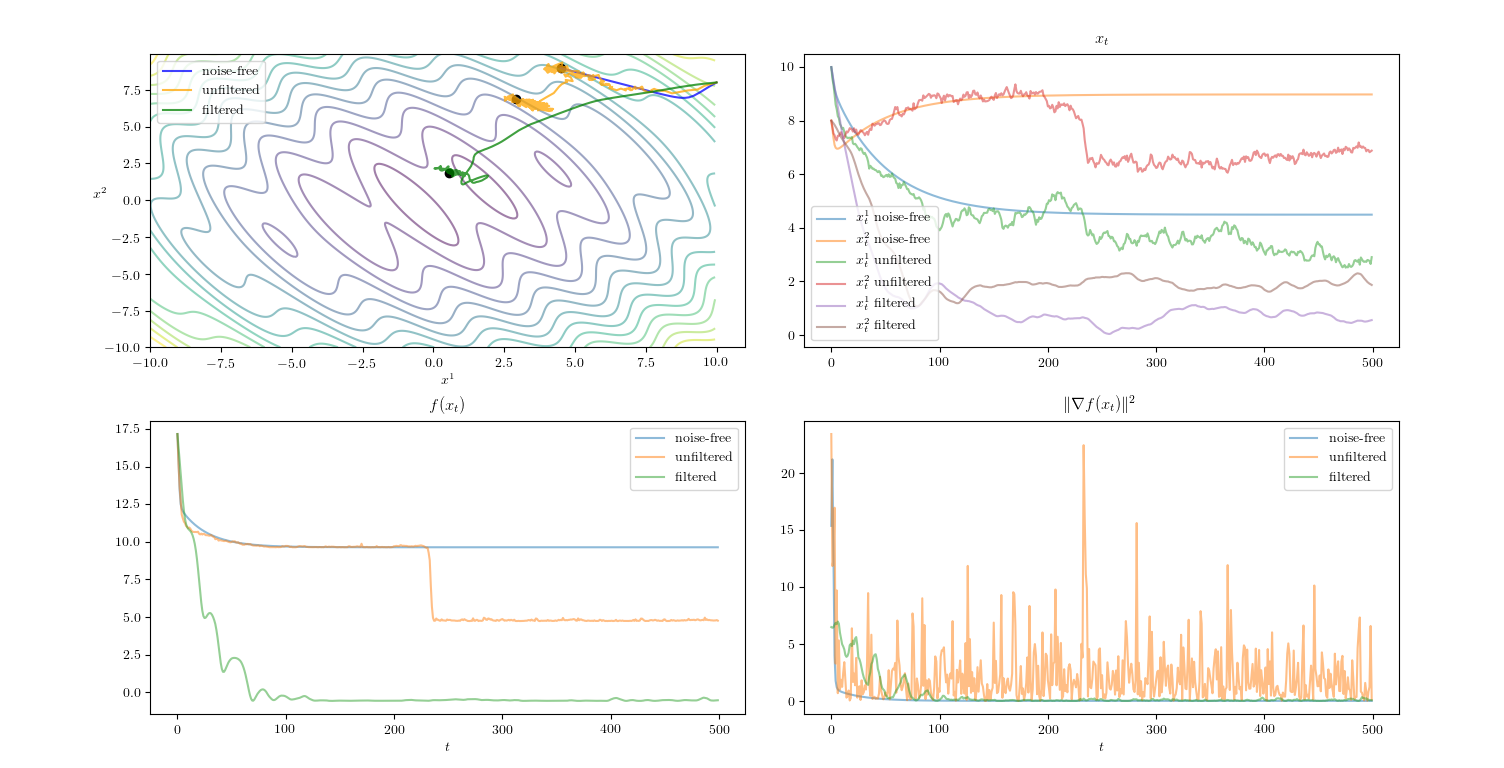}
		\caption{Optimization of the function $f(x^1,x^2)$ in \eqrref{eq:opt1} via the dynamics ordinary SGD algorithm. The stepsize was held constant at $\alpha_t=0.1$ and the algorithm was run for $T=500$ iterations.}
		\label{fig:opt1}
	\end{figure}

	\subsection{Parameter Fitting with Black Box Variational Inference}
	In this experiment, we approximated a two-dimensional target distribution with a two-dimensional Gaussian which has a diagonal covariance. Specifically we approximated
	\[
		p(z) \propto \exp\left(\phi\left(\frac{y}{1.35}\right) + \phi\left(\frac{x}{e^{y}}\right)\right),
	\]with $z=[x,y]$ where $\phi$ is the log-normal p.d.f., by the variational family
	\[
		q(z|\lambda)=\calN\left(z|\mu, \Sigma\right),~~~\lambda = (\mu,\Sigma)
	\]where $\mu\in\R^2,~\Sigma\in\mathcal{M}_{2\times 2}(\R)$ and $\Sigma$ is restricted to be diagonal.\newline

	We used Black Box Variational Inference (BBVI) \cite{ranganath2013black} to optimize a stochastic estimate of the evidence lower bound (ELBO)
	\[
		\calL(\lambda)=E_{q(\cdot|\lambda)}[\log p(x,z) - \log q(z|\lambda)]\approx \frac{1}{S}\sum_{i=1}^S [\log p(x, z_i) - \log q(z_i|\lambda)];~~~z_i\sim q(\cdot |\lambda)
	\]This objective function was maximized over $\lambda$ using with the gradients coming from backpropagation via the reparameterization trick \cite{kingma2015variational}. The gradient $\nabla_\lambda \calL$ is necessarily stochastic since it is a Monte Carlo approximation of an expectation. In our experiment, we used the extreme case of a single sample from $q(\cdot|\lambda)$ (i.e. $S=1$) to approximate the expectation. \newline 

	In Figure~\rref{fig:bbvi1}, we see a comparison between unfiltered and filtered results for the BBVI problem above. We used the RMSProp dynamics \eqrref{eq:rms0} to maximize $\calL(\lambda)$. The results illustrate the negative effect that high gradient variance can have on optimization performance.\newline

	Of course, one could improve the variance of the unfiltered dynamics by using more samples to estimate the objective; in this case, the performance of the filtered and unfiltered algorithms is similar. However, when a large number of samples is unavailable or impractical, filtering is a good way to improve performance.

	\begin{figure}[ht!]
	\centering
		\includegraphics[scale=0.5]{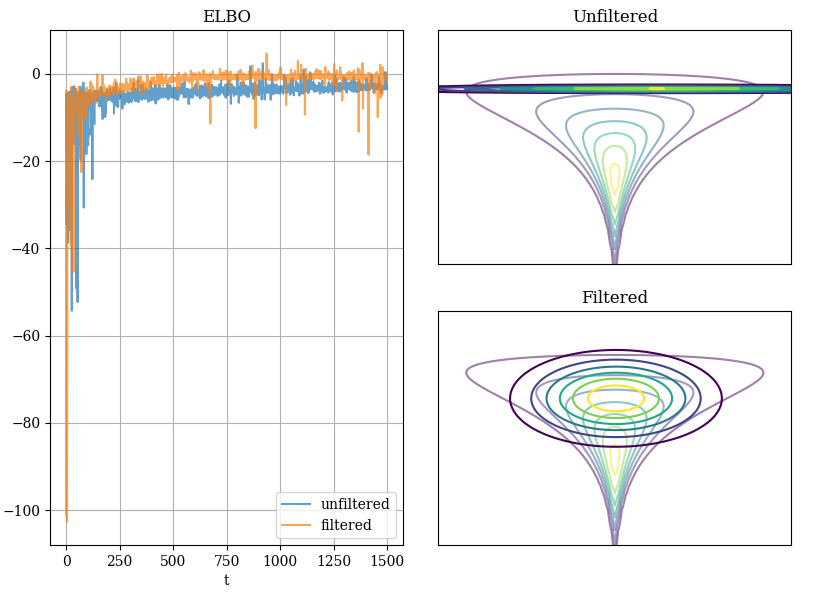}
		\caption{Using BBVI to fit a multivariate normal approximation to a non-Gaussian target. The algorithm uses the reparameterization trick \protect\cite{kingma2015variational} and the automatic differentiation Python package Autograd \protect\cite{autograd} to maximize the evidence lower-bound (ELBO) of the variational objective. The ELBO is estimated using a single sample from the variational approximation. The algorithm is run for $T=1500$ steps using RMSProp.}
		\label{fig:bbvi1}
	\end{figure}

	\subsection{Minibatch Neural Network Training}\label{sec:minbatchNN}
	In this experiment, we trained a simple multi-layer perceptron (MLP) to perform a 1-d regression task using minibatches to approximate the loss function. In particular, we used samples from $0.5\cos(x)$ which were corrupted by additive Gaussian noise and unequally sampled, and whose inputs were scaled and shifted. We used $N=80$ data points, and a batch-size of 8, with randomly sampled batches from the whole data set. We also tested two architectures, layer sizes (1,4,4,1) and (1,20,20,1), to study the effect of dimensionality on the performance of the algorithm. The former is a 33-dimensional problem, and the latter is 481-dimensional.\newline

	In Figure~\rref{fig:NNreg}, we see the improvement that filtering contributes to the optimization performance. In both cases, the filtered algorithm reached a local minimum in significantly fewer algorithm iterations (approx 7x for the small network and 2x for the large network). Both algorithms did converge to similar log-posteriors, and exhibited stability in their minima.

	\begin{figure}[ht!]
		\centering
		\begin{tabular}{c c}
		\includegraphics[scale=0.4]{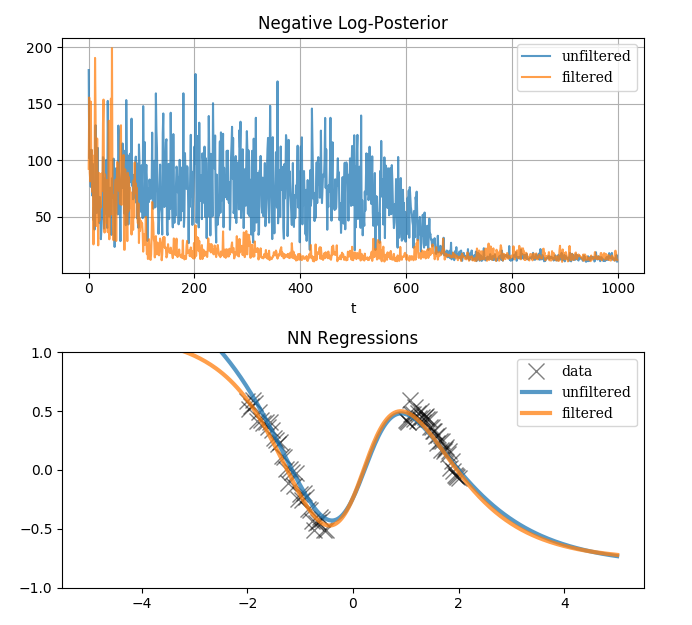} &
		\includegraphics[scale=0.4]{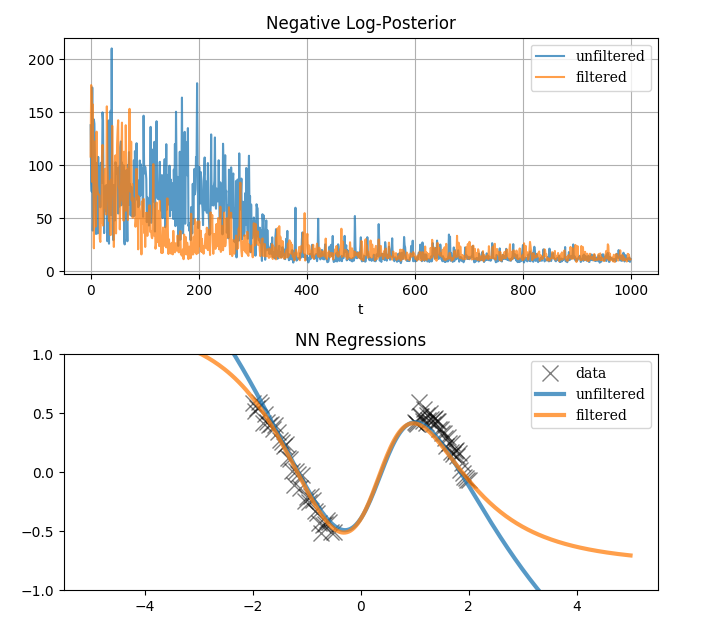}\\
		\textbf{(a)} & \textbf{(b)}
		\end{tabular}
		\caption{Comparison of MLP regression performance. All examples use 1000 iterations, a batch-size of 8 from a sample size of 80, and a stepsize $\alpha_t=0.01 \cdot 1.001^{-t}$. \textbf{(a)}: Using layer sizes (1,4,4,1). \textbf{(b)}: Using layer sizes (1,20,20,1).}
		\label{fig:NNreg}
	\end{figure}

	\subsection{MNIST Classifier}\label{subsec:mnist}
	We studied the use of the Distributed KGD algorithm from Section~\rref{sec:extensions} to classify MNIST digits \cite{lecun1998mnist} using the full-size dataset. We used a vanilla MLP classifier with layer sizes (748, 10, 10, 10). These small sizes were chosen to fit the computational limitations of the testing environment. Even with this relatively small network, training represents a 7,980-dimensional problem, over 10x the size of the problem in Section~\rref{sec:minbatchNN}. Hence, the Distributed KGD algorithm (Algorithm~\rref{alg:distkgd}) was required for this task to be completed in reasonable time.\newline

	We compared the regular and filtered Distributed RMSProp in two different regimes: small-batch (batch-size 32) and large-batch (batch-size 256). This allowed us to compare the performance of these algorithms in the presence of high- and low-noise gradients respectively. See Figure~\rref{fig:MNISTres} for the results.\newline

	We see that the Distributed KGD algorithm has equal optimization performance in the presence of low gradient variance, but significantly outperforms in the high-variance regime. In the former case, both algorithms achieve approx. 0.9 out-of-sample accuracy, which is reasonable considering the small size of the network. In the high-variance regime, the filtered optimization's accuracy plateaus significantly higher (approx 0.75) than the unfiltered version (approx 0.65) for the same number of steps. This suggests that the high variance in the gradient estimate caused the unfiltered optimization to converge to a \emph{worse} local minimum in terms of accuracy. This behaviour can also be seen in Figure~\rref{fig:opt1}.

	\begin{figure}[ht!]
		\centering
		\begin{tabular}{c c}
		\includegraphics[scale=0.3]{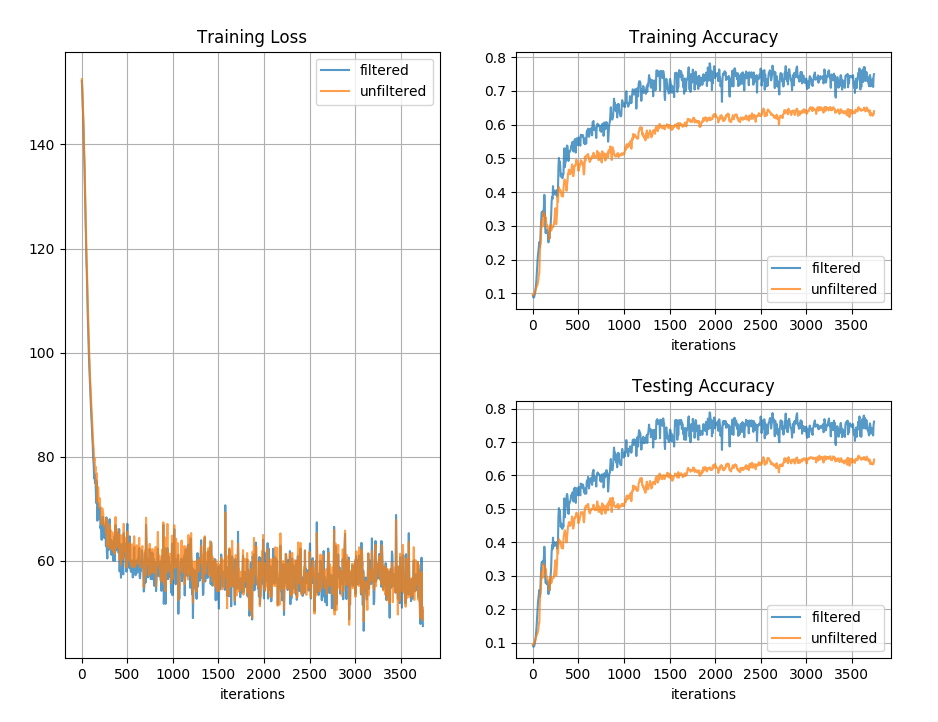} & 
		\includegraphics[scale=0.35]{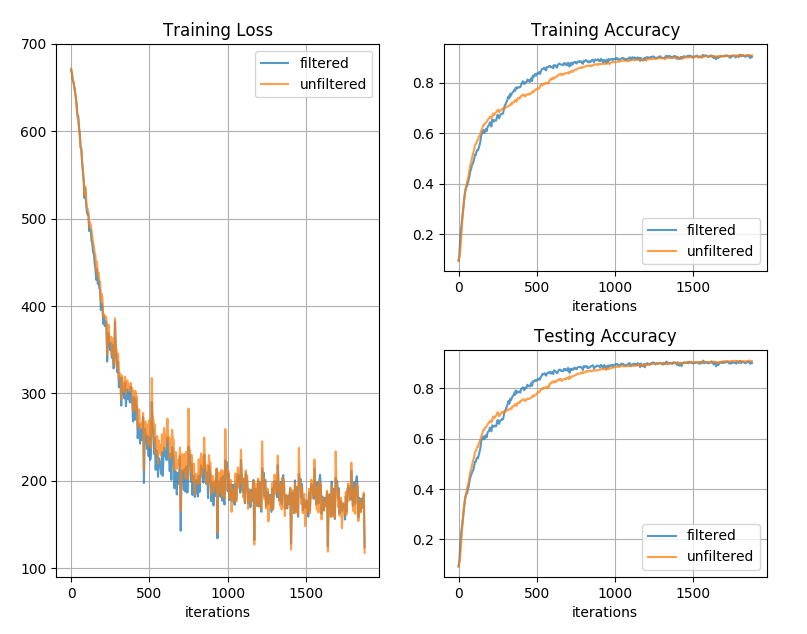}\\
		\textbf{(a)} & \textbf{(b)}
		\end{tabular}
		\caption{Comparison of MNIST classification performance. Both tests use a stepsize of 0.001, and $\gamma=0.9$ to train a (784, 10, 10, 10) network. The filtered version uses a sub-vector size of $D=50$. \textbf{(a)}: Batch-size 32 with 2 epochs (3,800 steps). \textbf{(b)}: Batch-size 256 with 8 epochs (1880 steps).}
		\label{fig:MNISTres}
	\end{figure}
	\section{Conclusions}\label{sec:conclusion}
	In this work, we have shown how to achieve superior performance in a variety of stochastic optimization problems through the application of Kalman filters to stochastic gradient descent. We provided a theoretical justification for the properties of these new stochastic optimization algorithms, and proposed methods for dealing with high-dimensional problems. We also demonstrated this algorithm's superior per-iteration efficiency on a variety of optimization and machine learning tasks.\newline

	\subsection*{Acknowledgments}
	We would like to thank Serdar Y\"uksel and Graeme Baker for their time and insightful feedback during the development this paper.


	\bibliographystyle{apacite}
	\bibliography{gradfiltering}
	\newpage

	\appendix
	\section{Proofs}
	\subsection{Filtering Proofs}\label{app:ltv}

	The Kalman filtering dynamics $\eqrref{eq:kal1}-\eqrref{eq:kal6}$ form a discrete-time LTV dynamical system, hence to study this system we will use the language and tools of LTV systems theory. In particular, the results below establish the asymptotic properties of the Kalman filter for general discrete-time LTV dynamical system

	\begin{equation}\label{eq:genlinsys}
		\rl{
			x_{t+1} & = A_t x_t + \Gamma_t w_t\\
			y_t & = C_t x_t + v_t
			};~~~x_t\in\R^m,~w_t\in\R^m,~y_t\in\R^p.
	\end{equation}
	where $w_t\sim \calN(0,Q_t)$ and $v_t\sim \calN(0, R_t)$ with $Q_t,R_t$ positive definite matrices.

	\begin{mydef}
	The for \eqrref{eq:genlinsys}, \textbf{observability matrix} is
	\[
		\calO(t, 1)=\sum_{i=1}^t \Phi^T(i,t)C_i^TR_i^{-1}C_i\Phi(i,t)
	\]and the \textbf{controllability matrix} is
	\[
		\calC(t,0)=\sum_{i=0}^{t-1} \Phi(t, i+1)\Gamma_i Q_{i}\Gamma_i^T\Phi^T(t, i+1)
	\]where $\Phi(t,s)$ is the state transition matrix of \eqrref{eq:genlinsys}, i.e. the solution to the homogeneous state equation. 
	\end{mydef}

	\begin{mydef}
	The Kalman filtered dynamics of \eqrref{eq:genlinsys} are \textbf{uniformly completely observable} if $\exists N_o\in\Z_{>0}$ and $\alpha,\beta\in\R_{>0}$ s.t. 
	\[
		0<\alpha I \leq \calO(t, t-N_o)\leq \beta I~~~\forall t\geq N_o
	\] and \textbf{uniformly completely controllable} if $\exists N_c\in\Z_{>0}$ and $\alpha',\beta'\in\R_{>0}$ s.t. 
	\[
		0<\alpha' I \leq \calC(t, t-N_c)\leq \beta' I ~~~\forall t\geq N_c.
	\]
	\end{mydef}
	Note that stability of the system's internal dynamics does not necessarily imply stability of the filter \cite{van2010nonlinear}. The below theorems, which are from \cite{jazwinski2007stochastic}, will be used to establish Theorem~\rref{thm:filt} in terms of uniform complete observability and controllability.\newline

	\begin{mythm}[Kalman Filter Stability]\label{thm:kfstab}
		If \eqrref{eq:genlinsys} is uniformly completely observable and uniformly completely controllable, then the discrete-time Kalman filter is uniformly asymptotically stable.
	\end{mythm}
	\begin{proof}
		See \cite{jazwinski2007stochastic} Theorem 7.4 p.240.
	\end{proof}

	\begin{mythm}[Kalman Filter Error Boundedness]\label{thm:kfbdd}
		If \eqrref{eq:genlinsys} is is uniformly completely observable and uniformly completely controllable, and if $P_0>0$ there exist $\alpha,\beta\in\R_{>0}$ and $N\in\Z_{>0}$ s.t. 
		\[
			\frac{\alpha}{1+\alpha\beta}I\leq P_{t|t}\leq \frac{1+\alpha\beta}{\alpha}I~~~\forall t\geq N.
		\]
	\end{mythm}
	\begin{proof}
		See \cite{jazwinski2007stochastic} Lemmas 7.1 \& 7.2 p.234.
	\end{proof}

	\begin{mythm}[Kalman Filter Robustness]\label{thm:kfrob}
		Suppose \eqrref{eq:genlinsys} is uniformly completely observable and uniformly completely controllable, and let $P^1_{t|t},P^2_{t|t}$ be two solutions to the Kalman filter dynamics with initial conditions $P^1_0,P^2_0\geq 0$ resp. Then $\exists k_1,k_2\in\R_{\geq 0}$ s.t. 
		\[
			\|P^1_{t|t}-P^2_{t|t}\|\leq k_1 e^{-k_2 t}\|P^1_0-P^2_0\|\to 0
		\]as $t\to \infty$.
	\end{mythm}
	\begin{proof}
		See \cite{jazwinski2007stochastic} Theorem 7.5 p.242.
	\end{proof}

	With the conditions for the stability of the Kalman filter above, we will show that Assumption~\rref{assump:alphamain} implies uniform complete observability and controllability for \eqrref{eq:linsys}, and this implies Theorem~\rref{thm:filt}.\newline

	\begin{mylemma}\label{lem:eig}
		Let $t\mapsto M_t$ be a $n\times n$, real, positive-definite matrix-valued function. Suppose that $\det(M_t)=K~\forall t$, and suppose that $\exists \lambda^*\in\R_{>0}$ such that $\lambda^i_t\leq \lambda^*~\forall t$ and every $i$, where $\lambda^i_t$ are the eigenvalues of $M_t$. Then $\exists \lambda_*\leq \lambda^i_t~\forall t$ and every $i$. 
	\end{mylemma}
	\begin{proof}
		Assume WLOG that $\det(M_t)=1~\forall t$. Note that each $\lambda^i_t>0$ and also that $\det(M_t)=\prod_{i=1}^n \lambda^{i}_t$. Then consider making $\lambda^n_t$ the smallest possible for any choice of $\lambda^1_t,\dots,\lambda^{n-1}_t$ under the constraint that $\det(M_t)=1$. This obviously occurs when $\lambda^1_t = \cdots = \lambda ^{n-1}_t=\lambda^*$. In this case, 
		\[
			\lambda^n_t = \frac{1}{\lambda^1_t\cdots\lambda ^{n-1}_t}=\frac{1}{(\lambda^*)^{d-1}}.
		\]Hence $\lambda^i_t\geq (\lambda^*)^{-(d-1)}=\lambda_*$.
	\end{proof}

	\begin{mylemma}\label{lem:linalg}
	\hfill
		\begin{enumerate}[(a)]
			\item For an invertible $n\times n$ matrix $M$ which satisfies $M^TM<\alpha I$ we have
			\[
				 (M^n)^T(M^n)<\alpha^n I
			\]
				\item Suppose that $U,V$ are two $n\times n$ positive definite matrices s.t. $0<U<V$. Then we have
			\[
				U^{-1}>V^{-1}.
			\]
			\item For a square, invertible matrix $M$, $(MM^T)^{-1}=(M^{-1})^TM^{-1}$.
			\end{enumerate}
	\end{mylemma}
	\begin{proof}
	\hfill
		\begin{enumerate}[(a)]
			\item \cite{lax2014linear} Ch 10.1 p.146.
			\item \cite{lax2014linear} Theorem 2, Ch 10.1 p.146
			\item Easy to directly verify.
		\end{enumerate}
	\end{proof}

	\begin{mylemma}[KGD Observability and Controllability]\label{lem:ctrlobs}
		Under Assumption~\rref{assump:alphamain}, the KGD dynamics \eqrref{eq:linsys} are uniformly completely observable and controllable.
	\end{mylemma}

	\begin{proof}
		This proof is organized into three steps: First, we prove uniform complete controllability, then extend the machinery of complete controllability, and finally apply this extended machinery to prove uniform complete observability.
		\begin{enumerate}
			\item We will take $N_c=2$. Then $\Gamma_t\equiv I$, $Q_{t}=\sigma_Q I>0~\forall t$ in the definition above, hence we have
		\[
			\calC(t,t-2)=\sigma_QI + \sigma_Q \Phi(t, t-1)\Phi^T(t, t-1)=\sigma_QI + \sigma_Q A_{t-1}A_{t-1}^T.
		\]It suffices to prove the case when $\sigma_Q=1$, and to prove the uniform boundedness of the second term. We will first show that $v^TA_{t-1}A_{t-1}^Tv=\|A_{t-1}^Tv\|^2$ has upper bound which does not depend on $t$ for an arbitrary $v\in\R^{2n}$. We have that
		\[
			A_{t-1}^T = \begin{bmatrix}
				I & 0 \\
				-\alpha_{t-1} & I
			\end{bmatrix}
		\] hence writing $v=[v_1^T,v_2^T]^T$ and using $\|x+y\|^2\leq 2\|x\|^2+2\|y\|^2$, we have that
		\begin{align*}
			\|A_{t-1}^Tv\|^2&=\|v_1\|^2 + \|v_2 - \alpha_{t-1} v_1\|^2\\
			&\leq \|v_1\|^2 +2\|v_2\|^2 +2(\alpha_{t-1})^2 \|v_1\|^2\\
			&=(1+2\alpha_{t-1}^2)\|v_1\|^2 + 2\|v_2\|^2\\
			&\leq (1+2\alpha^*)\|v_1\|^2 + 2\|v_2\|^2.
		\end{align*}
		where $\alpha^*=\max_t \alpha_t<\infty$ by assumption. To show the existence of a uniform lower bound, we will appeal to Lemma~\rref{lem:eig} to show that the spectrum of $\Phi(t,t-1)\Phi(t,t-1)^T$ is uniformly lower-bounded, and this implies the required matrix inequality. Indeed, $\det(A_{t-1})=1$ so that 
		\begin{equation}
			\det(\Phi(t,t-1)\Phi(t,t-1)^T)=\det(A_{t-1}A_{t-1}^T)=1
		\end{equation}
		and since we have just shown an uniform upper matrix bound (which implies a uniformly upper-bounded spectrum), Lemma~\rref{lem:eig} implies there is a uniform lower matrix bound which is $>0$.
		\item 
		Using the multiplicative property from Lemma~\rref{lem:linalg}(a), we can iterate the result for the controllability matrix by conjugating $A_{t-T}$ with $A_{t-T+1},\dots,A_{t-1}$ to see that $\Phi(t,t-T)\Phi(t,t-T)^T$ is also uniformly bounded above for a fixed $T>0$. The fact that we are conjugating by matrices whose determinants are all 1 allows us to use Lemma~\rref{lem:eig} again to conclude the existence of a uniform lower bound on $\Phi(t,t-T)\Phi(t,t-T)^T$. Now,
		\begin{equation}\label{eq:iterctrl}
			\Phi(t-T,t)^T\Phi(t-T,t)= (\Phi(t, t-T)^{-1})^T(\Phi(t, t-T)^{-1}) = \left(\Phi(t,t-T)\Phi(t,t-T)^T\right)^{-1}
		\end{equation}
		by Lemma~\rref{lem:linalg}(c). Then Lemma~\rref{lem:linalg}(b) implies the existence of uniform bounds on $\Phi(t-T,t)^T\Phi(t-T,t)$. 
		 \item Now, for the observability matrix, consider for a fixed $N_o>0$
		\[
			\calO(t, t-N_o)=\sigma_R^{-1}\sum_{i=t-N_o}^t \Phi^T(i,t)C_i^TC_i\Phi(i,t).
		\]We will also assume that $\sigma_R=1$. Thus, by ensuring that $C_i=I_{2n\times 2n}$ at least once every $N_o$ timesteps, one of the matrices in the sum $\calO(t, t-N_o)$ of the form \eqrref{eq:iterctrl} and hence is uniformly bounded above and below. The other terms in the sum are not definite since $C_i$ is rank-deficient at those times, but they are positive and uniformly bounded, hence the system is uniformly completely observable.
		\end{enumerate}
	\end{proof}

	Note that since $N_o$ is arbitrary, we can make it can be very large. In practice, $N_o$ can be larger than the total number of timesteps of a run of the algorithm. This is why we have omitted this technical detail from the setup in Section~\rref{sec:setup}.\newline

	\begin{mylemma}[Extended KGD Observability and Controllability]\label{lem:extobsctrl}
	 Under Assumptions~\rref{assump:alphamain} and \rref{assump:regmain}, if $\mu_t=\mu$ in \eqrref{eq:mom1} and $\gamma_t=\gamma$ in \eqrref{eq:rms1} updates are constant, then the dynamics for \eqrref{eq:mom1}-\eqrref{eq:mom2} and \eqrref{eq:rms1}-\eqrref{eq:rms2} are uniformly completely controllable and observable.
	\end{mylemma}
	\begin{proof}
		We need to prove a version of Lemma~\rref{lem:ctrlobs}. This lemma uses the facts that $A_tA_t^T<\alpha I$ for some $\alpha$ uniformly in $t$, and that the determinant of $A_t$ is constant. The latter follows immediately by inspection since $A_t$ is upper triangular. For the uniform upper-bound, consider first momentum: for $v=[v_1^T, v_2^T,v_3^T]^T\in\R^{3n}$ we have

		\begin{align*}
			\|A_{t-1}^T v\|^2 &=\|v_1\|^2 + \|\alpha_{t-1} \mu v_1 + \mu v_2\|^2 + \|-\alpha_{t-1}(1-\mu)v_1 - (1-\mu)v_2 + v_3\|^2\\
			&\leq \|v_1\|^2 + 2(\alpha_{t-1}\mu)^2\|v_1\|^2 + 2\mu^2 \|v_2\|^2 + 3(\alpha_{t-1}(1-\mu))^2 \|v_1\|^2 + 3(1-\mu)^2\|v_2\|^2 + 3\|v_3\|^2\\
			&=[1+2(\alpha_{t-1} \mu)^2 + 3(\alpha_{t-1}(1-\mu))^2]\|v_1\|^2 + [2\mu^2 + 3(1-\mu)^2]\|v_2\|^2 + 3\|v_3\|^2\\
			&\leq [1+2(\alpha^* \mu)^2 + 3(\alpha^*(1-\mu))^2]\|v_1\|^2 + [2\mu^2 + 3(1-\mu)^2]\|v_2\|^2 + 3\|v_3\|^2
		\end{align*}

		with $\alpha^*=\max_{t} \alpha_t < \infty$. \newline

		For RMSProp, we proceed similarly:

		\begin{align*}
			\|A_{t-1}^T v\|^2&=\|v_1\|^2 + \|\gamma v_2\|^2 + \|-\alpha_{t-1}\diag(\beta_{t-1})v_1 + (1-\gamma)\diag(\nabla f(x_{t-1}))v_2 + v_3\|^2 \\
			&\leq \|v_1\|^2 + \gamma^2\|v_2\|^2 + 3\alpha_{t-1}^2 \underbrace{\|\diag(\beta_{t-1})v_1\|^2}_{(*)} +  3(1-\gamma)^2 \underbrace{\|\diag(\nabla f(x_{t-1}))v_2\|^2}_{(**)}  + 3\|v_3\|^2
		\end{align*}
		It remains to show that the terms $(*)$ and $(**)$ are uniformly bounded above. In the former case, this is true since $\beta_t = (\sqrt{r_{t+1}}+\veps)^{-1}$ and $\sqrt{r_{t+1}}\geq 0$. In the latter case, using the Frobenius norm on the matrix $\diag(\nabla f(x_{t-1}))$ we have
		\[
			\|\diag(\nabla f(x_{t-1}))v_3\| \leq \|\diag(\nabla f(x_{t-1}))\|\|v_3\|=\|\nabla f(x_{t-1})\|\|v_3\|\leq G\|v_3\|
		\]where $G$ is the  uniform bound on $\nabla f$.
		Hence the rest of the proof of Lemma~\rref{lem:ctrlobs} can be applied to the momentum and RMSProp cases respectively.
	\end{proof}

	\subsection{Stochastic Optimization Proofs}\label{app:opt}
	For simplicity, to prove the convergence of \eqrref{eq:kgd} we prove convergence of the system 
	\begin{equation}\label{eq:kgdsimple}
		\rl{
			v_{t+1} &= (1-\beta)v_t + \beta g(x_t;\xi_t)\\
			x_{t+1} &= x_t - \alpha_t v_{t+1}
		}
	\end{equation}
	with $x\in\R^n$, $\beta\in]0,1[$ and appeal to Proposition~\rref{prop:reduction} which implies that we do not lose any generality in doing so.\newline

	\begin{myprop}\label{prop:reduction}
	\hfill
	\begin{enumerate}[(a)]
		\item If $\beta$ from \eqrref{eq:kgdsimple} is allowed to vary with time s.t. $0<\beta_*<\beta_t<\beta^*<1$ then the results of this section hold with $\beta_*$ and $\beta^*$ replacing $\beta$ appropriately.
		\item We can replace $\beta$ from \eqrref{eq:kgdsimple} by a positive-definite matrix $B$ with $0<\beta_* I< B<\beta^*I < I$ without materially affecting the results of this section.
	\end{enumerate}
	\end{myprop}
	\begin{proof}
	\hfill
		\begin{enumerate}[(a)]
			\item Since we never used the property $(1-\beta)+\beta=1$, we can replace each instance of $\beta$ by the appropriate bound $\beta^*$ or $\beta_*$ depending on the sign of $\beta$ to preserve any inequality, and all finite sums/ products can be bounded by appropriate upper/lower bounds.
			\item The matrix $B$ (or possibly $I-B$ as the case may be) defines a new inner product which is re-weighted by $B$. By the equivalence of norms (hence inner products) on $\R^n$, the claim holds.
		\end{enumerate}
	\end{proof}
			

	In the sequel, $E[\cdot]$ will denote the expectation w.r.t. the joint distribution of the $\xi_s,~s\leq t$ where $t$ is the latest time in the expectation. For convenience, we will restate the necessary assumptions from Section~\rref{sec:analysis} as needed.

	\begin{assump}[Same as Assumption~\rref{assump:regmain}]\label{assump:reg}
		The objective function $f:\R^n\to \R$ is $C^3$, with uniformly bounded 1st, 2nd, and 3rd derivatives. In particular $\nabla f$ is Lipschitz with constant $L$.
	\end{assump}

	\begin{assump}[Same as Assumption~\rref{assump:gmain}]\label{assump:g}
		The random variables $g(x_t,;\xi_t)$ satisfy the following properties
		\begin{enumerate}[(a)]
			\item $E[\nabla f(x_t)^T g(x_t;\xi_t)]\geq \mu E[\|\nabla f(x_t)\|^2]$ for some constant $\mu>0$ and $\forall t$;
			\item $E[\|g(x_t;\xi_t)\|^2]\leq M + M_GE[\|\nabla f(x_t)\|^2]$ for constants $M,~M_G>0$ $\forall t$.
		\end{enumerate}
	\end{assump}

	\begin{mylemma}\label{lem:taylor}
		Under Assumption~\rref{assump:reg}, 
		\begin{enumerate}[(a)]
			\item The there is $H^*>0$ s.t. $H(x)<H^* I$ for every $x$; and
			\item $\nabla f$ has the Taylor expansion about $a\in\R^n$
		\[
			\nabla f(x) = \nabla f(a) + H(a)\cdot(x-a) + \veps_a(x)
		\]with $\veps_a(x)\in \R^n$ and moreover, $\veps_a(x)^T(x-a)\leq C \|x-a\|^3$ for some $C>0$.
		\end{enumerate}
	\end{mylemma}
	\begin{proof}
	\hfill
		\begin{enumerate}[(a)]
			\item The entries of $H(x)$ are uniformly bounded, hence so are its eigenvalues, which implies the result.
			\item Let $g(x):=\nabla f(x)$ and denote its components $g=[g^1,\dots,g^n]^T$. Each $g^k:\R^n\to \R,~k=1,\dots,n$ and is $C^2$. The first-degree Taylor expansion for a component function $g^k$ about a point $a\in\R^n$ is \cite{driver2004analysis}
			\begin{equation}
				g^k(x) = \sum_{|\alpha|\leq 1} \frac{D^\alpha g^k(a)}{\alpha!}(x-a)^\alpha + \sum_{|\beta|=2}R_\beta(x)(x-a)^\beta;~~~D^\alpha g^k(a)=\frac{\partial^{|\alpha|} g^k}{\partial x_1^{\alpha_1}\cdots\partial x_n^{\alpha_n}}(a)
			\end{equation}
			with the multiindex conventions $\alpha=(\alpha_1,\dots,\alpha_n)$, $|\alpha|=\alpha_1+\cdots+\alpha_n$, $\alpha!=\alpha_1!\cdots\alpha_n!$, and $x^\alpha=x_1^{\alpha_1}\cdots x_n^{\alpha_n}$; and with
			\[
				R_\beta(x) = \frac{|\beta|}{\beta!}\int_0^1 (1-t)^{|\beta|-1}D^\beta g^k(a+t(x-a))\dd t.
			\] Note that 
			\[
				\{D^\alpha g^k(a)\st |\alpha|=1\}=\{\partial_i g^k(a)\st i=1,\dots,n\}~~\text{and}~~\{D^\beta g^k(x)\st |\beta|=2\}=\{\partial^2_{ij}g^k(a)\st i,j=1,\dots,n\}
			\]so for $\beta$ having nonzero components $i,j$ we can write
			\[
				R^k_{ij}(x) := R^k_\beta(x) =\frac{|\beta|}{\beta!}\int_0^1 (1-t)^{|\beta|-1}\partial^2_{ij} g^k(a+t(x-a))\dd t.
			\]Hence we may re-write Taylor expansion in vector form as
			\[
				g^k(x) = g^k(a) + \nabla g^k(a)^T(x-a) + \underbrace{(x-a)^T R^k(x)(x-a)}_{\veps^k_a(x)}.
			\]where $R^k(x)=[R^k_{ij}(x)]_{ij}$ . Hence we can write the vector-valued version of Taylor's expansion by stacking the expressions above component-wise to get
			\[
				g(x) = g(a) + J_g(x)(x-a)+\veps_a(x)
			\]where $\veps_a(x)\in\R^n$ and $J_g$ is the Jacobian matrix of $g$. Since $g(x)\equiv\nabla f(x)$, it is clear that $J_g(x)=H(x)$. We need to show that $\veps_a(x)^T(x-a)\leq C \|x-a\|^3$. To see this, note that 
			\[
				\partial^2_{ij} g^k(x) = \partial^3_{ijk}f(x)<c_0
			\]for some $c_0>0$ and for every $x\implies R^k(x)<c_1 I$ for some $c_1>0$.  Hence consider 
			\begin{align*}
				\veps_a(x)^T(x-a) &= \sum_{k=1}^n \veps^k_a(x)(x-a)_k\\
				&= \sum_{k=1}^n (x-a)^TR^k(x)(x-a)\cdot (x-a)_k\\
				&\leq \sum_{k=1}^n c_1 (x-a)^T(x-a) \cdot (x-a)_k\\
				&= c_1 \|x-a\|^2 \sum_{k=1}^n (x-a)_k\\
				&\leq c_1 \|x-a\|^2 \|x-a\|_1\leq c_1 \|x-a\|^2 \cdot c_2\|x-a\|_2=c_1 c_2 \|x-a\|^3
			\end{align*}
			since $\|\cdot\|_1\leq c_2\|\cdot\|_2$ for some $c_2>0$ by the equivalence of norms in $\R^n$.
		\end{enumerate}
	\end{proof}

	\begin{mylemma}\label{lem:bounded}
		If $E[\|g(x_t;\xi_t)\|^3]\leq G^3$ for $G>0$ then there is a $V>0$ s.t. $E[\|v_t\|^3]\leq V^3$ $\forall t$; moreover $E[\|v_t\|^2]\leq V^2$. 
	\end{mylemma}
	\begin{proof}
		Note first that $v_{t+1}$ is a dynamical system with state transition $v_s\to v_{t+1}$ given by
		\[
			v_{t+1} = (1-\beta)^{t-s+1}v_s + \beta\sum_{r=s}^t(1-\beta)^{t-r}g(x_r, \xi_r).
		\]Using the triangle inequality and Young's inequality, we can derive $\|x+y\|^3\leq(\|x\| + \|y\|)^3\leq 4\|x\|^3 + 4\|y\|^3$ and applying this to the transition formula above for $s=0$, we get
		\begin{align*}
			\|v_{t+1}\|^3&\leq 4(1-\beta)^{3(t+1)}\|v_0\|^3 + 4\beta^3 \underbrace{\left\|\sum_{r=0}^t (1-\beta)^{t-r}g(x_r;\xi_r)\right\|^3}_{(*)}.
		\end{align*}
		$v_0$ is fixed and deterministic, so the first term has bounded expectation. We analyze $(*)$ as follows: let $S_t = \sum_{r=0}^t (1-\beta)^r$ for $t>0$ and $S_0:=1$. Note that $S_t\to S_\infty<\infty$ since $(1-\beta)<1$. We then have
		\begin{align*}
			\left\|\sum_{r=0}^t (1-\beta)^{t-r}g(x_r;\xi_r)\right\|^3 &= \left\|\sum_{r=0}^t (1-\beta)^{r}g(x_{t-r};\xi_{t-r})\right\|^3\\
			&= \left\|S_t\cdot \frac{1}{S_t}\sum_{r=0}^t (1-\beta)^{r}g(x_{t-r};\xi_{t-r})\right\|^3\\
			&= (S_t)^3 \left\|\sum_{r=0}^t \frac{(1-\beta)^{r}}{S_t}g(x_{t-r};\xi_{t-r})\right\|^3.
		\end{align*}
		The terms $\{(1-\beta)^r/S_t\}_{r=0}^t$ sum to 1 and $\|\cdot\|^3$ is convex so we may apply Jensen's inequality to obtain
		\begin{align*}
			(*)=\left\|\sum_{r=0}^t (1-\beta)^{t-r}g(x_r;\xi_r)\right\|^3& \leq (S_t)^3 \sum_{r=0}^t \frac{(1-\beta)^{r}}{S_t}\left\|g(x_{t-r};\xi_{t-r})\right\|^3\\
			&= (S_t)^2 \sum_{r=0}^t (1-\beta)^{r}\left\|g(x_{t-r};\xi_{t-r})\right\|^3.
		\end{align*}
		Hence taking expectations and using the assumption that $E[\|g(x_s;\xi_s)\|^3]\leq G^3~\forall s\geq 0$ we have
		\begin{align*}
			E[(*)]&\leq (S_t)^2 \sum_{r=0}^t(1-\beta)^r E[\|g(x_{t-r};\xi_{t-r})\|^3]\\
			&\leq (S_t)^2 \cdot G^3\cdot \sum_{r=0}^t(1-\beta)^r\\
			&\leq (S_\infty)^3\cdot G^3<\infty.
		\end{align*}
		The fact that this bound is independent of $t$ implies that this bound is uniform $\forall t$. To conclude, using H\"older's inequality we have that $E[\|v_t\|^2]\leq E[\|v_t\|^3]^{2/3}\leq (V^3)^{2/3}=V^2$.
	\end{proof}

	\begin{myprop}\label{prop:descentdir}
		Assume that Assumption~\rref{assump:reg} and Assumption~\rref{assump:g} hold, and that ${E[\nabla f(x_0)^Tv_0]\geq 0}$. Then $\forall t\geq 0$
		\[
			E[\nabla f(x_t)^Tv_{t+1}]\geq \beta\mu E[\|\nabla f(x_t)\|^2] - \sum_{s=0}^{t-1} (1-\beta)^{t-s}\left[\alpha_{s} H^*V^2 + \alpha_{s}^2 CV^3\right].
		\]
	\end{myprop}
	\begin{proof}
		We proceed by induction on $t$. For $t=0$:
		\begin{align*}
			E[\nabla f(x_0)^Tv_1] &= E[\nabla f(x_0)^T((1-\beta)v_0 + \beta g(x_0;\xi_0))]\\
			&=(1-\beta) E[\nabla f(x_0)^T v_0] + \beta E[\nabla f(x_0)^T g(x_0;\xi_0)]\\
			&\geq \beta E[\nabla f(x_0)^T g(x_0;\xi_0)]\\
			&\geq \beta\mu E[\|\nabla f(x_0)\|^2]
		\end{align*}
		and the claim holds for $t=0$. Hence assume that the result holds for $t-1$, i.e. that
		\[
			E[\nabla f(x_{t-1})^Tv_{t}]\geq \beta\mu E[\|\nabla f(x_{t-1})\|^2] - \sum_{s=0}^{t-2} (1-\beta)^{t-1-s}\left[\alpha_{s} H^*V^2 + \alpha_{s}^2 CV^3\right].
		\]Then for $t$ we have that
		\begin{align*}
			E[\nabla f(x_t)^Tv_{t+1}]&= E[\nabla f(x_t)^T((1-\beta)v_t + \beta g(x_t; \xi_t))]\\
			&=\underbrace{(1-\beta)E[\nabla f(x_t)^Tv_t]}_{(1)}+\underbrace{\beta E[\nabla f(x_t)^Tg(x_t,\xi_t)]}_{(2)}.
		\end{align*}
		Working with $(1)$ and using the Taylor expansion for $\nabla f$ from Lemma~\rref{lem:taylor} (with $H_t:=H(x_t)$) we have
		\begin{align*}
			(1) &= (1-\beta)E[(\nabla f(x_{t-1})+H_{t-1}(x_t-x_{t-1})+\veps_{x_{t-1}}(x_{t}))^T v_t]\\
			&=(1-\beta)E[(\nabla f(x_{t-1})^Tv_t+(x_t-x_{t-1})^TH_{t-1}^Tv_t+\veps_{x_{t-1}}(x_{t})^Tv_t] \\
			&= (1-\beta)\left\{E[\nabla f(x_{t-1})^Tv_t] -\alpha_{t-1} E[v_t^TH_{t-1}^T v_t] + \alpha_{t-1}^2E[\veps_{x_{t-1}}(x_t)^Tv_t]\right\}\\
			&\geq (1-\beta)\left\{E[\nabla f(x_{t-1})^Tv_t] -\alpha_{t-1} E[v_t^TH_{t-1}^T v_t] - \alpha_{t-1}^2|E[\veps_{x_{t-1}}(x_t)^Tv_t]|\right\}
		\end{align*}
		Lemma~\rref{lem:taylor} implies that $\veps_{x_{t-1}}(x_t)^Tv_t\leq C\|v_t\|^3$, and Lemma~\rref{lem:bounded} implies that $E[\|v_t\|^3]\leq V^3$. Then using Assumption~\rref{assump:g}(a) on $(2)$ and combining with the previous inequality we have
		\begin{align*}
			E[\nabla f(x_t)^Tv_{t+1}]&\geq \beta \mu E[\|\nabla f(x_t)\|^2] +  (1-\beta)\left\{E[\nabla f(x_{t-1})^Tv_t] -\alpha_{t-1} H^*E[\|v_t\|^2] - \alpha_{t-1}^2CE[\|v_t\|^3]\right\} \\
			&\geq  \beta \mu E[\|\nabla f(x_t)\|^2] +  (1-\beta)\left\{E[\nabla f(x_{t-1})^Tv_t] -\alpha_{t-1} H^*V^2 - \alpha_{t-1}^2CV^3\right\} .
		\end{align*}
		Substituting the inductive hypothesis, we have
		\begin{align*}
			E[\nabla f(x_t)^Tv_{t+1}]&\geq  \beta \mu E[\|\nabla f(x_t)\|^2] +  (1-\beta)\left\{\beta\mu E[\|\nabla f(x_{t-1})\|^2] - \sum_{s=0}^{t-2} (1-\beta)^{t-1-s}\left[\alpha_{s} H^*V^2 + \alpha_{s}^2 CV^3\right] \cdots \right.\\
			&\left.\cdots -\alpha_{t-1} H^*V^2 - \alpha_{t-1}^2CV^3\right\}\\
			&= \beta \mu E[\|\nabla f(x_t)\|^2] + (1-\beta)\beta\mu E[\|\nabla f(x_{t-1})\|^2] \cdots \\
			&\cdots - (1-\beta)\sum_{s=0}^{t-2} (1-\beta)^{t-1-s}\left[\alpha_{s} H^*V^2 + \alpha_{s}^2 CV^3\right] -(1-\beta)[\alpha_{t-1} H^*V^2 + \alpha_{t-1}^2CV^3]\\
			&\geq  \beta \mu E[\|\nabla f(x_t)\|^2] - (1-\beta)\sum_{s=0}^{t-2} (1-\beta)^{t-1-s}\left[\alpha_{s} H^*V^2 + \alpha_{s}^2 CV^3\right] \cdots\\
			&\cdots -(1-\beta)[\alpha_{t-1} H^*V^2 + \alpha_{t-1}^2CV^3]\\
			&= \beta \mu E[\|\nabla f(x_t)\|^2] - \sum_{s=0}^{t-2} (1-\beta)^{t-s}\left[\alpha_{s} H^*V^2 + \alpha_{s}^2 CV^3\right] -(1-\beta)[\alpha_{t-1} H^*V^2 + \alpha_{t-1}^2CV^3]\\
			& = \beta \mu E[\|\nabla f(x_t)\|^2] - \sum_{s=0}^{t-1} (1-\beta)^{t-s}\left[\alpha_{s} H^*V^2 + \alpha_{s}^2 CV^3\right]
		\end{align*}
		where we have used first that $(1-\beta)\beta\mu E[\|\nabla f(x_{t-1})\|^2]\geq 0$ to neglect it from the RHS, and then also the fact that $(1-\beta)[\alpha_{t-1} H^*V^2 + \alpha_{t-1}^2CV^3]$ is the $s=t-1$ term in the sum. Hence by induction the claim holds $\forall t\geq 0$.
	\end{proof}

	\begin{myprop}\label{prop:varcomp}
		Under Assumptions~\rref{assump:reg} and~\rref{assump:g}, we have
		\[
			E[\|v_{t+1}\|^2]\leq \wtilde{M} + \wtilde{M_G}E\|\nabla f(x_t)\|^2]
		\]for constants $\wtilde{M},~\wtilde{M}_G>0$. 
	\end{myprop}
	\begin{proof}
		\begin{align*}
			E[\|v_{t+1}\|^2]&=E[\|(1-\beta) v_t + \beta g(x_t,\xi_t)\|^2]\\
			&\leq 2 (1-\beta)^2 E[\|v_t\|^2] + 2\beta^2 E[\|g(x_t,\xi_t)\|^2]\\
			&\leq 2(1-\beta)^2V^2 + 2\beta^2(M + M_G E[\|\nabla f(x_t)\|^2]).
		\end{align*}
	\end{proof}

	\begin{mylemma}\label{lem:ineq}
		Under Assumptions~\rref{assump:reg} and \rref{assump:g}, we have that 
		\begin{align*}
			E[f(x_{t+1}) - f(x_t)]&\leq -(\wtilde{\mu}-\frac{1}{2}\alpha_t L\wtilde{M}_G)\alpha_t \|\nabla f(x_t)\|^2 + \frac{1}{2}\alpha_t^2 L\wtilde{M} +\alpha_t \sum_{s=0}^{t-1} (1-\beta)^{t-s}\left[\alpha_{s} H^*V^2 + \alpha_{s}^2 CV^3\right]
		\end{align*}
		where $\wtilde{\mu}=\beta\mu$, and $L$ is the Lipschitz constant of $\nabla f$.
	\end{mylemma}
	\begin{proof}
		First, the Lipschitz property of $\nabla f$ implies
		\[
			f(x)\leq f(x) + \nabla f(y)^T(x-y)+\frac{L}{2}\|x-y\|^2~\forall x,y\in\R^n.
		\]Then using this, we have
		\begin{align*}
			f(x_{t+1}) - f(x_t) &\leq \nabla f(x_t)^T(x_{t+1}-x_t) + \frac{1}{2}L\|x_{t+1}-x_t\|^2\\
			&=-\alpha_t \nabla f(x_t)^Tv_{t+1} + \frac{1}{2}L\alpha_t^2\|v_{t+1}\|^2.
		\end{align*}
		Then taking expectations of both sides and using Propositions \rref{prop:descentdir} and \rref{prop:varcomp} we have
		\begin{align*}
			E[f(x_{t+1}) - f(x_t)] &\leq -\alpha_t E[\nabla f(x_t)^Tv_{t+1}] + \frac{1}{2}\alpha_t^2 E[\|v_{t+1}\|^2]\\
			&\leq -\wtilde{\mu}\alpha_tE[\|\nabla f(x_t)\|^2] +\alpha_t \sum_{s=0}^{t-1} (1-\beta)^{t-s}\left[\alpha_{s} H^*V^2 + \alpha_{s}^2 CV^3\right] + \frac{1}{2}\alpha_t^2L(\wtilde{M} + \wtilde{M_G}E[\|\nabla f(x_t)\|^2])\\
			& = -(\wtilde{\mu}-\frac{1}{2}\alpha_tL \wtilde{M}_G)\alpha_t E[\|\nabla f(x_t)\|^2] + \frac{1}{2}\alpha_{t}^2 L\wtilde{M} +\alpha_t \sum_{s=0}^{t-1} (1-\beta)^{t-s}\left[\alpha_{s} H^*V^2 + \alpha_{s}^2 CV^3\right]
		\end{align*}
	\end{proof}

	\begin{assump}[Same as Assumption~\rref{assump:alphamain}]\label{assump:alpha}
		$\{\alpha_t\}_{t=0}^\infty\subset\R$ is non-increasing, $\alpha_t>0~\forall t$, ${\sum \alpha_t= \infty,~\sum \alpha_t^2<\infty}$.
	\end{assump}

	\begin{mythm}
		Suppose Assumptions~\rref{assump:reg},~\rref{assump:g},~and~\rref{assump:alpha} hold, and that $f_*=\min_{x\in\R^n} f(x)$ exists. Then
		\[
			\lim_{T\to \infty}E\left[\sum_{t=0}^T\|\nabla f(x_t)\|^2\right]<\infty.
		\]
	\end{mythm}
	\begin{proof}
		Assume WLOG that $\alpha_0L \wtilde{M}_G \leq \wtilde{\mu} $. Then using Lemma~\rref{lem:ineq}
		\begin{align*}
			&f_*-E[f(x_0)] \leq E[f(x_{T+1})] - E[f(x_0)] = \sum_{t=0}^T (E[f(x_{t+1})]-E[f(x_t)])\\
			&\leq \sum_{t=0}^T \left\{-(\wtilde{\mu}-\frac{1}{2}\alpha_t L\wtilde{M}_G)\alpha_t \|\nabla f(x_t)\|^2 + \frac{1}{2}\alpha_t^2 L\wtilde{M} +\alpha_t \sum_{s=0}^{t-1} (1-\beta)^{t-s}\left[\alpha_{s} H^*V^2 + \alpha_{s}^2 CV^3\right]\right\}\\
			&\leq \sum_{t=0}^T \left\{-\frac{1}{2}\wtilde{\mu}\alpha_t E[\|\nabla f(x_t)\|^2]+ \frac{1}{2}\alpha_t^2 L\wtilde{M} +\alpha_t \sum_{s=0}^{t-1} (1-\beta)^{t-s}\left[\alpha_{s} H^*V^2 + \alpha_{s}^2 CV^3\right]\right\}\\
		\end{align*}
		hence 
		\[
			\sum_{t=0}^T \alpha_t E[\|\nabla f(x_t)\|^2]\leq \frac{2(E[f(x_0)]-f_*)}{\wtilde{\mu}} + \frac{L\wtilde{M}}{\wtilde{\mu}}\sum_{t=0}^T\alpha_t^2 + \frac{2}{\wtilde{\mu}}\sum_{t=0}^T \alpha_t \sum_{s=0}^{t-1} (1-\beta)^{t-s}\left[\alpha_{s} H^*V^2 + \alpha_{s}^2 CV^3\right].
		\]Hence we need to show that the last term is summable as $T\to \infty$. 
		Since $\alpha_s^2\to 0$ faster than $\alpha_s$, it is sufficient to show that the term containing $\alpha_s$ is summable. Hence we show that
		\[
			\lim_{T\to \infty}\sum_{t=1}^T \sum_{s=0}^{t-1} \alpha_t \alpha_s(1-\beta)^{t-s}<\infty
		\] where the $t=0$ term is zero since it is the empty sum. Let us exchange the order of summation
		\[
			\sum_{t=1}^T \sum_{s=0}^{t-1} \alpha_t\alpha_s (1-\beta)^{t-s}= \sum_{s=0}^{T-1} \sum_{t=s+1}^T\alpha_s\alpha_t (1-\beta)^{t-s}.
		\]Then using the fact that $\alpha_t$ is decreasing $\implies \alpha_s\geq \alpha_t$ for $s\leq t$, we have
		\begin{align*}
			\sum_{s=0}^{T-1} \sum_{t=s+1}^T\alpha_s\alpha_t (1-\beta)^{t-s}&\leq \sum_{s=0}^{T-1} \alpha_s^2 \sum_{t=s+1}^T (1-\beta)^{t-s}\\
			&= \sum_{s=0}^{T-1} \alpha_s^2 \sum_{t=0}^{T-s-1} (1-\beta)^{t+1}\\
			&=\frac{1-\beta}{\beta}\sum_{s=0}^{T-1} \alpha_s^2 \cdot [1-(1-\beta)^{T-s-1}]\\
			&\leq \frac{1-\beta}{\beta} \cdot C\sum_{s=0}^{T-1} \alpha_s^2 < \infty\text{ as }T\to \infty.
		\end{align*}
	\end{proof}

	\begin{mycoro}\label{coro:result}
		Under Assumptions~\rref{assump:reg},~\rref{assump:g}, and~\rref{assump:alpha}, we have
		\[
			\liminf_{t\to \infty}E[\|\nabla f(x_t)\|^2]=0.
		\]
	\end{mycoro}
	\begin{proof}
		Following \cite{bertsekas2000gradient}, if not, then $\exists\epsilon>0$ s.t. $E[\|\nabla f(x_t)\|^2]>\epsilon$ $\forall t\geq 0$. But then
		\[
			\sum_{t=0}^t \alpha_tE[\|\nabla f(x_t)\|^2]\geq \epsilon\sum_{t=0}^T\alpha_t=\infty,
		\]contradicting the preceding Theorem.
	\end{proof}

	\subsection{Miscellaneous Proofs}\label{app:misc}
	Consider the ``heavy ball'' momentum update \cite{qian1999momentum} 
	\begin{equation}\label{eq:hb}
		\rl{
			u_{t+1} &= \beta u_t -\alpha \nabla f(x_t)\\
			x_{t+1} &= x_t + u_{t+1}
		}
	\end{equation}
	and the exponentially smoothed gradient descent update which is a simplified version of KGD
	\begin{equation}\label{eq:ema}
		\rl{
			v_{t+1} &= (1-\beta)v_t + \beta \nabla f(x_t)\\
			x_{t+1} &= x_t - \alpha v_{t+1}.
		}
	\end{equation}
	where $0<\beta<1$,~$\alpha>0$. We claim that these two methods are not via a change of parameters.
	\begin{myprop}
		There is are no parameters $\wtilde{\alpha}>0,~0<\wtilde{\beta}<1$ for \eqrref{eq:ema} which will transform \eqrref{eq:ema} into \eqrref{eq:hb}.
	\end{myprop}
	\begin{proof}
		Suppose the contrary. Then we would have
		\[
			\rl{
				x_{t+1} &= x_t + \beta u_t -\alpha \nabla f(x_t)\\
				x_{t+1} &= x_t -\wtilde{\alpha}((1-\wtilde{\beta})v_t +\wtilde{\beta}\nabla f(x_t))
			}\implies \rl{
				\beta &= -\wtilde{\alpha}(1-\wtilde{\beta})\\
				\alpha &= \wtilde{\alpha}\wtilde{\beta}
			}
		\]by comparing like terms. The condition on $\beta$ is clearly impossible, hence we have a contradiction.
	\end{proof}
	\newpage
	\section{Additional Imagery}\label{app:imagery}
	\begin{figure}[ht!]
		\centering
		\begin{tabular}{c c}
			\includegraphics[scale=0.4]{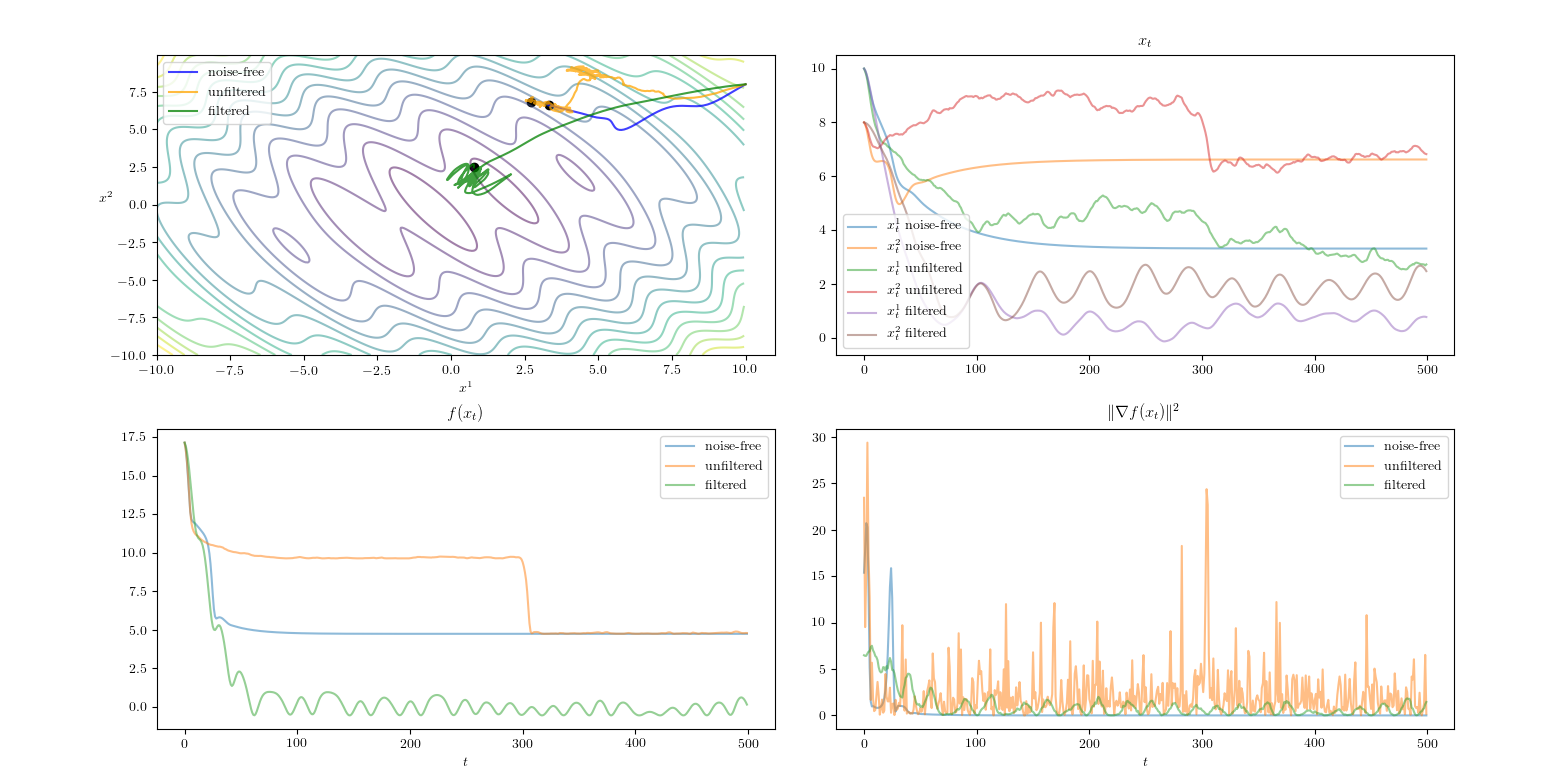}\\
			\includegraphics[scale=0.4]{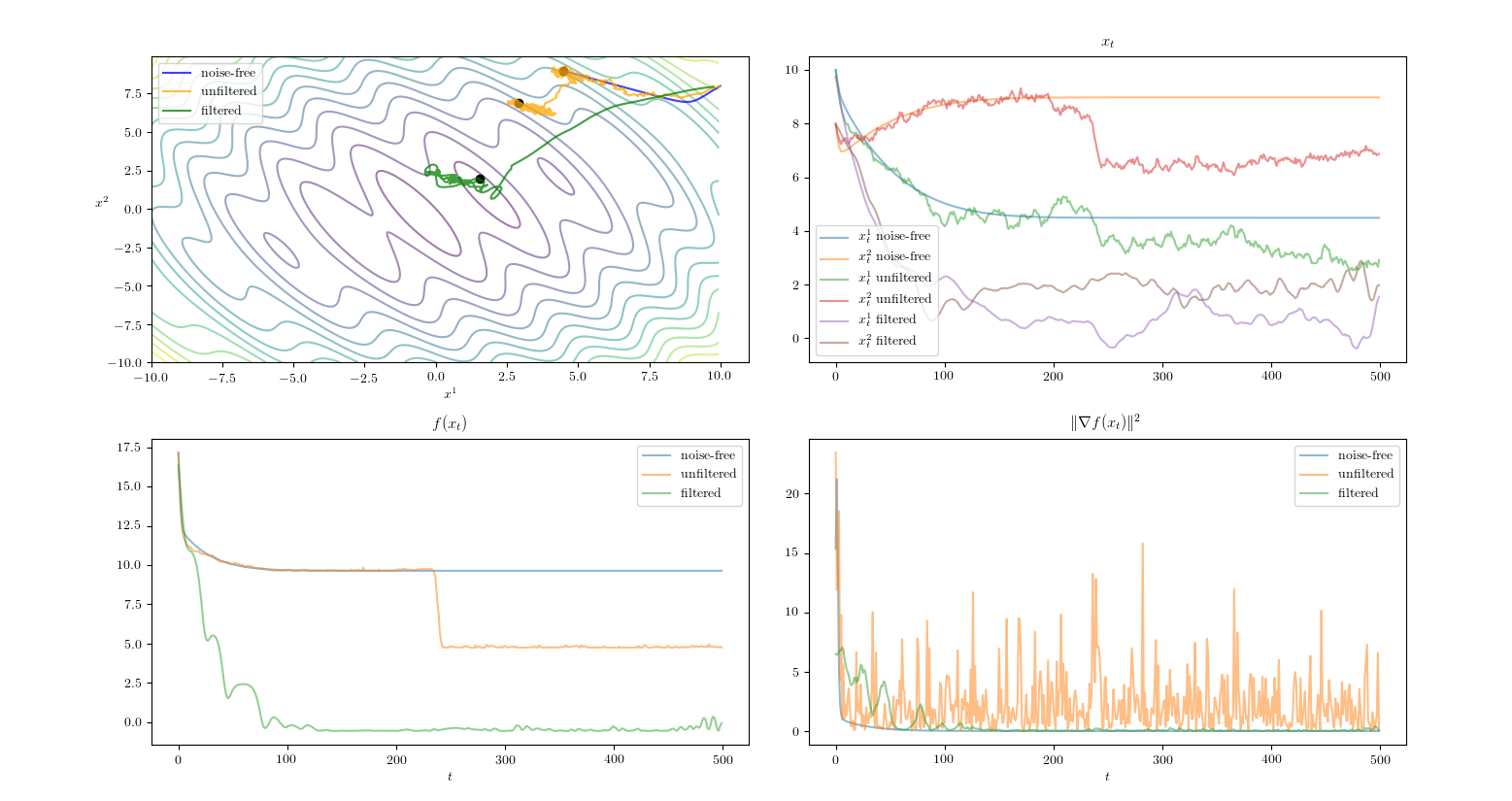}
		\end{tabular}
		\caption{Additional tests for filtered stochastic minimization from Section~\rref{sec:exp1}.}
		\label{fig:addtl}
	\end{figure}

\end{document}